\documentclass{article}

\usepackage{microtype}
\usepackage{graphicx}

\usepackage{epsfig}
\usepackage{algorithm}
\usepackage{algorithmic}
\usepackage{multirow}

\usepackage{bm}

\usepackage{mathrsfs}

\usepackage{enumerate}

\usepackage[labelformat=simple]{subcaption}

\usepackage{booktabs} 

\usepackage{hyperref}

\usepackage{amsmath}
\usepackage{amssymb}
\usepackage{mathtools}
\usepackage{amsthm}

\usepackage[capitalize,noabbrev]{cleveref}

\usepackage[textsize=tiny]{todonotes}

\newtheorem{theorem}{Theorem}
\newtheorem{lem}{Lemma}
\newtheorem{rem}{Remark}
\newtheorem{prop}{Proposition}
\newtheorem{cor}{Corollary}

\newtheorem{ass}{Assumption}

\newcounter{hypA}

\newcounter{hypB}

\newcounter{hypC}

\usepackage{graphicx}
\usepackage[normalem]{ulem}

\definecolor{reddarker}{rgb}{0.36, 0.0, 0.0}

\makeatother

\newcommand{\calX}{\mathcal{X}}

\newcommand{\bbE}{\mathbb{E}}

\usepackage{xspace}
\usepackage{tabu}
\usepackage{booktabs}

\DeclareMathOperator*{\argmax}{argmax}

\makeatletter

\begin{document}

\title{Accelerating Look-ahead in \\ Bayesian Optimization: \\
Multilevel Monte Carlo is All you Need}




\author{Shangda Yang$^1$, Vitaly Zankin$^1$, Maximilian Balandat$^3$, 
Stefan Scherer$^3$, \\
Kevin Carlberg$^3$, Neil Walton$^{1,2}$, and Kody J.~H. Law$^{1,3}$}
\date{
$^1$Department of Mathematics, University of Manchester, Manchester, M13 9PL, UK\\
$^2$Durham University Business School, Millhill Lane, \\
Durham, DH1 3LB, UK\\
$^3$Meta Platforms, Inc. 1 Hacker Way, Menlo Park, CA, 94025, USA\\[2ex]
\today}

\vskip 0.3in
\maketitle

\begin{abstract}
We leverage multilevel Monte Carlo (MLMC) to improve the performance of multi-step look-ahead Bayesian optimization (BO) methods that involve nested expectations and maximizations. 
Often these expectations must be computed by Monte Carlo (MC).
    The complexity rate of naive MC degrades for nested operations, whereas
    MLMC is capable of achieving the canonical MC convergence rate
    for this type of problem,
    independently of dimension and without any smoothness assumptions.
    {Our theoretical study focuses on the approximation improvements for two- and three-step look-ahead acquisition functions, 
    but, as we discuss,} the approach is generalizable in various ways, including beyond the context of BO.     
    Our findings are 
    verified numerically and the benefits of MLMC for BO are illustrated on several
    benchmark examples.
    Code is available at 
    \href{https://github.com/Shangda-Yang/MLMCBO}
    {\small \texttt{https://github.com/Shangda-Yang/MLMCBO}}.
   \end{abstract}

\section{Introduction}
Bayesian optimization (BO)
is a global optimization method for expensive-to-evaluate black-box functions that generally have unknown structures. 
In this setting, only the function value is observed for a given input value. 
BO works by constructing a probabilistic surrogate model, often a Gaussian process (GP), 
for the black-box function and then iteratively updating the surrogate model, 
guided by the optimization of an 
acquisition function, until some stopping criterion is achieved.
{See \cite{frazier2018tutorial,shahriari2015taking} for a detailed review of BO.}

The choice of acquisition function is crucial in the design of BO algorithms. 
Typical myopic acquisition functions such as upper confidence bound \cite{auer2002finite,auer2002using,srinivas2009gaussian}, expected improvement (EI) \cite{movckus1975bayesian,jones1998efficient}, 
entropy search \cite{hennig2012entropy}, and predictive entropy search \cite{hernandez2014predictive} 
only consider the immediate reward of the decision. 
{They do not trade off the depth of decision-making with the computational budget. }
{By contrast, a multi-step acquisition function looks into the future and will optimize over the planning horizon by formulating the problem as a Markov decision process (MDP). 
This can substantially reduce the number of required function evaluations 
\cite{gonzalez2016glasses,wu2019practical,yue2020non,lee2020efficient,jiang2020binoculars,jiang2020efficient}.
However, the complexity of approximating such acquisition functions 
increases exponentially with the number of look-ahead steps
for current methods, which limits their use.}

{\bf The goal of this paper is to improve the computational complexity of approximating look-ahead acquisition functions.}

Multi-step look-ahead acquisition functions have traditionally been optimized by nesting optimization and 
Monte Carlo (MC) estimations, 
whereas we leverage a 
nested Sample Average Approximation (SAA) \cite{balandat2020botorch}, 
which facilitates the application of deterministic higher-order optimization methods. 
Standard MC (without nested operations) requires $O(\varepsilon^{-2})$ samples to achieve mean-squared-error (MSE) of $\varepsilon$.
For just one nested operation, the sample complexity of MC with SAA is at least $O(\varepsilon^{-3})$ and $O(\varepsilon^{-4})$ in the worst case, depending on the smoothness of the integrand. 
For $k$ nested operations, the cost for MC grows exponentially in $k$ to  $O(\varepsilon^{-2(k+1)})$, leading to a {\em curse of dimensionality}.

This paper is the first to apply the \textbf{MLMC} framework to Bayesian Optimization, thereby
{\em improving the performance of look-ahead methods}.
MLMC for expectations of approximations, such as nested MC, 
works by constructing a telescoping sum of estimators from low accuracy to high accuracy \cite{giles2015multilevel}. 
Computational complexity is reduced by performing most simulations with low-accuracy and low-cost,  
and only very few high-accuracy and high-cost simulations (see Figure \ref{fig:Cost}). 
Investigations of MLMC with nested MC appeared in 
\cite{giles2019multilevel} for efficient risk estimation inside a Heaviside function,  \cite{hironaka2020multilevel, giles2019decision} for the expected value of sample information with the feasible set being finite, \cite{goda2020multilevel} for nested MC inside a log function, and \cite{giles2018mlmc} for more general problems.
\cite{hu2021bias} use MLMC gradient methods for computing the nested MC problems, and 
\cite{beck2020multilevel} use MLMC 
for approximating the expected information gain in a Bayesian optimal experimental design.
MLMC is still a nascent technology within Machine Learning, with results only now beginning to appear 
in the context of Markov chain Monte Carlo \cite{lyne2015russian,strathmann2015unbiased,chada2022multilevel,cai2022multi} and stochastic gradient descent \cite{fujisawa2021multilevel, goda2022unbiased}.

\begin{figure}[ht!]
    \centering
         \centering
        \includegraphics[width=0.8\textwidth]{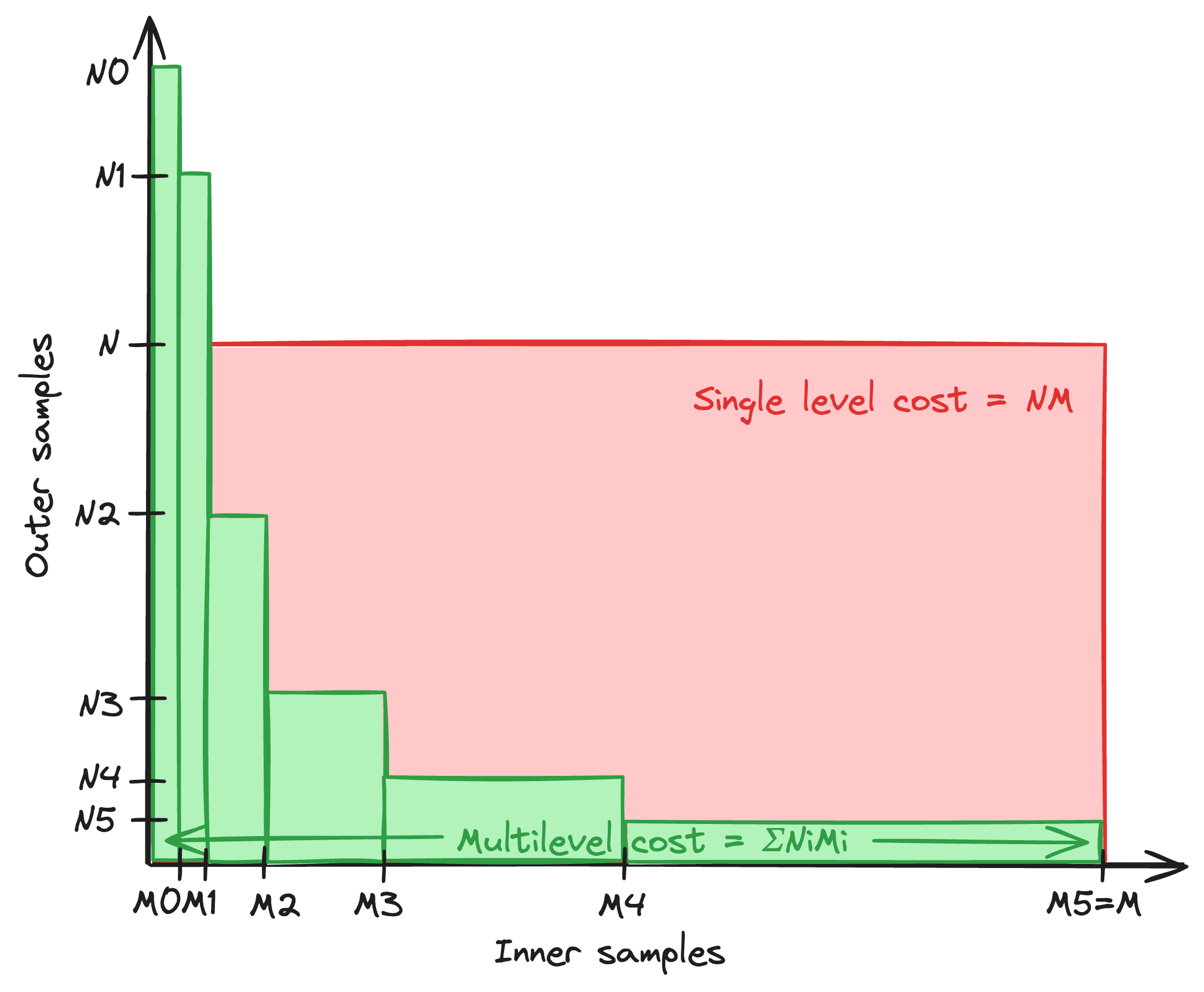}
        \vspace{-2ex}
        \captionsetup{width=1.18\linewidth}
         \caption{A graphical description of MLMC complexity (green area) improvement over nested MC (red area).}
         \label{fig:Cost}
     \end{figure}
    
{Variance reduction techniques are crucial for look-ahead acquisition functions in BO. 
Recent work using variance reduction techniques for myopic acquisition functions include}
\cite{bogunovic2016truncated, balandat2020botorch, lee2020efficient, nguyen2017predictive}. 
In general, for smooth and relatively low-dimensional integrals, quadrature methods rule
\cite{heath2018scientific,james1980monte}. 
However, these methods suffer from the 
{curse of dimensionality},
{i.e., the {\em rate} of convergence degrades exponentially with dimension.
Sparse grid approaches can mitigate this effect in moderate dimensions \cite{bungartz2004sparse}.
However, the only methods that deliver dimension-independent convergence 
rates for individual integrals are MC methods \cite{james1980monte}.
For smooth functions, quasi-MC (QMC) methods can improve the canonical MC rate; 
however, for non-smooth functions, the rate returns to the canonical MC rate \cite{caflisch1998monte}.
Importance sampling can reduce variance, 
but the variance {\em constant} often grows exponentially in the dimension \cite{chatterjee2018sample}.}
Furthermore, the critical dimension-dependence of the convergence rate
{\em reappears along the nesting axis}, i.e. here the look-ahead direction,
where one has expectation of function of expectation of...
To our knowledge, only MLMC methods, 
or the closely related {\em multifidelity} methods \cite{peherstorfer2018survey},
are capable of delivering the 
MC rate of convergence (independently of dimension) for nested MC approximation of non-smooth functions. 
Look-ahead acquisition functions in BO are indeed non-smooth, 
hence their approximation is a prime candidate for 
leveraging MLMC in Machine Learning.

{We apply 
MLMC to compute a multi-step look-ahead acquisition function. 
The feasible set is infinite, and the integrand may contain nested optimization functions due to the MDP formulation.}  
We prove that the error of the SAA with nested MC approximations can be decomposed into ``variance" and ``bias" terms, which allows us to apply MLMC for improved efficiency. 
{Through its ability to improve complexity,
our results suggest that MLMC is the state-of-the-art method for accelerating 
current Bayesian optimization frameworks.}

The narrative is summarized concisely as follows:
\begin{itemize}
\item[(i)] The primary bottleneck in BO is quantified by the number of black-box {\em function evaluations} (BB), 
which can be substantially reduced by look-ahead acquisition functions (AF).
\item[(ii)] These look-ahead acquisition functions are {\em themselves expensive to approximate}.
The cost of (BB), $B$, implies a cost constraint for (AF), say between $(B/10,B)$.
\item[(iii)] 
We deliver a method which can be tuned 
to get designs which are either the best 
for a given budget or cheapest for a target accuracy, in terms of {\em order of complexity},
i.e. {\em the bigger the budget or the higher the target accuracy, the more gain there is to be had.}
\end{itemize}

This article is structured as follows. 
Section~\ref{subsec:BO:Overview} illustrates the intuition and benefits of applying the MLMC method in BO.
Section~\ref{subsec:BO:Construction} introduces the general BO settings for this paper. Section~\ref{sec:saa} discusses SAA, standard MC, and nested MC. 
Basic MLMC and MLMC for BO is introduced 
in Section~\ref{sec:mlmc}. 
Numerical tests are conducted in Section \ref{sec:num}, 
where {we illustrate the benefits of the method on several benchmark examples from \cite{balandat2020botorch}.
}

\section{Bayesian Optimization}\label{sec:BO}
{Suppose we want to maximize an expensive-to-evaluate black-box function $g:\mathcal X \rightarrow \mathbb R$,
for $\mathcal{X}\subseteq\mathbb R^d$ where $d$ is the dimension of the input space.} Here, ``black box'' means that we do not know the structure or the derivatives of the function. We can only observe the output of the function given an input value. Mathematically, for an input value $x^{(i)}$, we observe
\begin{equation*}
    y^{(i)} = g(x^{(i)}) + \epsilon_i\qquad 
\end{equation*}
where the $\epsilon_i$ are i.i.d. zero mean Gaussian random variables.

Suppose we have collected $n$ observations $\mathcal{D}_n := \{(x^{(i)},y^{(i)})\}_{i=1}^{n}$. We can solve the optimization problem sequentially using BO, which involves emulating $g(x)$ using a surrogate model, specifically a Gaussian Process, selecting $x^{(n+1)}$ given $\mathcal{D}_n$ by maximizing an acquisition function, incorporating $ \{(x^{(n+1)},y^{(n+1)})\}$ into the surrogate, and iterating.  The standard setting for BO is low/moderate dimension $d$ and small sample size $n$. 

\subsection{Brief overview of intuition and result}\label{subsec:BO:Overview}

We illustrate the intuition and result concisely with a simple 1D toy example first,
where 
$\mathcal{X}=[-10,10]$ and
\begin{equation}\label{eq:toy}
     g(x) = e^{-(x - 2)^2} + e^{-\frac{(x - 6)^2}{10}} + \frac{1}{x^2 + 1},
\end{equation}
which has one global maximum $g(x^*= 2.0087) = 1.4019$, shown in 
Figure \ref{fig:ToyExample}(a). 

\begin{figure}[ht!]
    \centering
    \vspace{-0.3cm}
     \begin{subfigure}[b]{0.6\columnwidth}
         \centering
        \includegraphics[width=\textwidth]{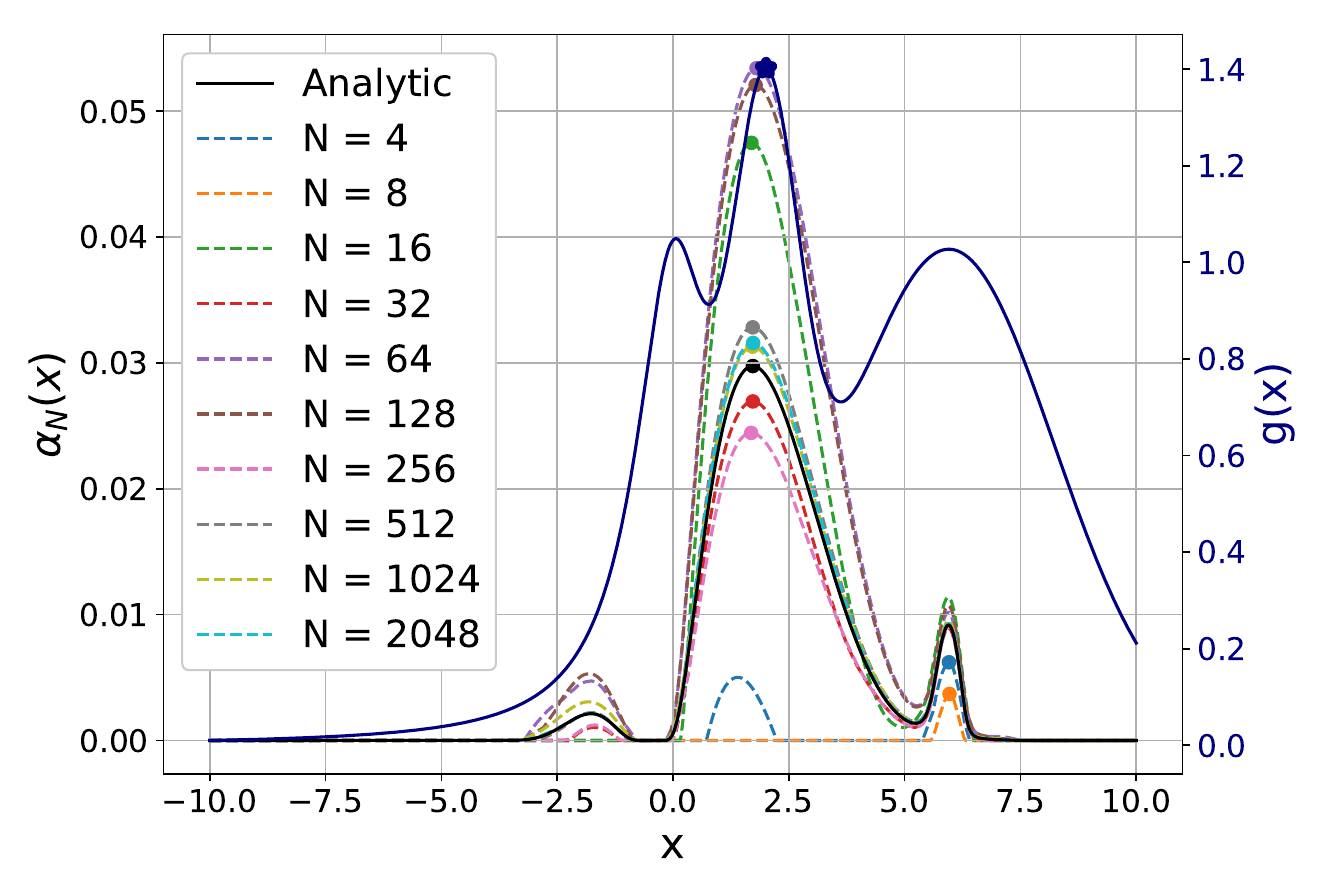}
        \captionsetup{width=\linewidth}
        \caption{}
         \label{fig:1DToy}
     \end{subfigure}
     \begin{subfigure}[b]{0.6\columnwidth}
         \centering
        \includegraphics[width=\textwidth]{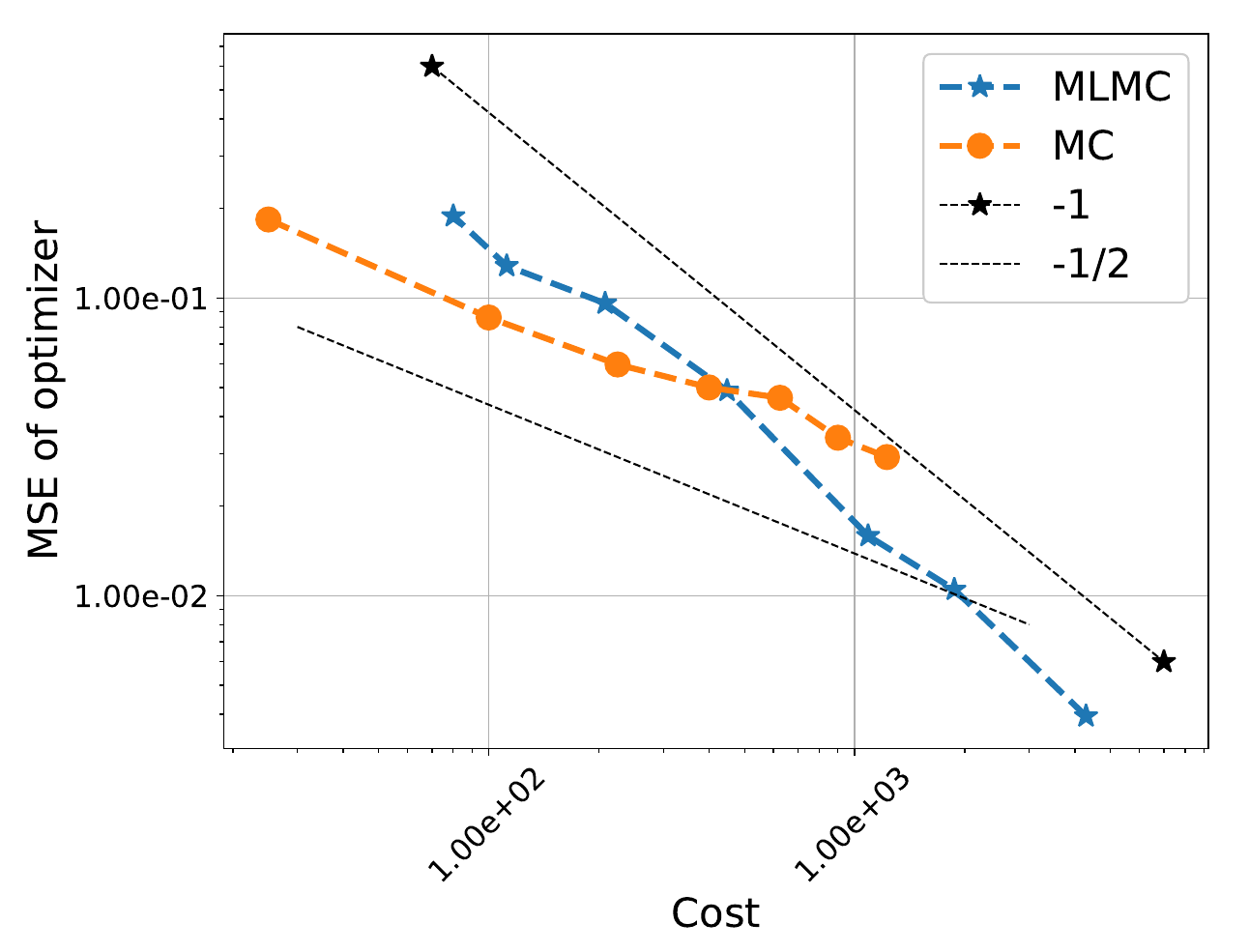}
        \captionsetup{width=\linewidth}
        \caption{}
         \label{fig:1DMSE}
     \end{subfigure}
    \caption{
    \textit{Panel \subref{fig:1DToy}:} 
     (i) The blue solid line is the objective function $g$ (with axis on the right). The function has a unique global maximizer and maximum. (ii) The black solid line is the analytical EI acquisition function, and the black dot is the reference solution. (iii) The dashed colored lines are the Monte Carlo approximation of the acquisition function with varying $N$, and the corresponding dots are the respective maximums. Low-accuracy (small $N$) approximations can result in maximizer of the approximation to be far from the true maximizer.  
    \textit{Panel \subref{fig:1DMSE}:} 
    Complexity diagram of 
    MLMC and nested MC approximation  
    of two-step look-ahead EI with the cost measured by the number of operations. 
    The reference solution for MSE is computed with high accuracy. Each curve is computed with 200 realizations. 
    }
    \label{fig:ToyExample}
\end{figure}

Figure \ref{fig:1DToy} illustrates the need to approximate the acquisition functions with high accuracy, because a low accuracy function approximation is susceptible to finding the wrong mode. Multi-modality can lead to a higher inaccuracy than one would expect in a uni-modal setting, and ultimately drive the BO trajectory off course, 
thus wasting costly function evaluations and delaying convergence.

In the following, we introduce a technique that improves accuracy or reduces computational complexity compared to the standard MC method while balancing exploration and exploitation found with BO.
In particular, we 
leverage the MLMC method to efficiently approximate analytically intractable 
look-ahead acquisition functions. 

{We use mean-squared error (MSE) as the error metric for acquisition function approximation. Figure \ref{fig:1DMSE} shows the MSE from approximating one acquisition function with MC and MLMC in BO. It shows that by applying MLMC, we can approximate the acquisition function's maximizer with the canonical convergence rate, 
{ i.e., Cost $\propto$ MSE$^{-1}$ rather than the sub-optimal Cost $\propto$ MSE$^{-2}$ rate of nested MC}.
Intuitively, this means that for a
given computational cost,
the maximizer can be approximated with higher accuracy, thus improving the performance of the whole BO algorithm. 
Alternatively, we can achieve the same accuracy as the MC method with less computational cost, reducing the computational cost of the BO algorithm. 
{The benefits of MLMC are more significant if high accuracy is required or if the approximation is very costly.}
{
Later, we will see that MLMC achieves a better normalized MSE for the same computational cost 
for the full outer BO problem across a range of examples (Figure~\ref{fig:fullBOMLqEI}).}

\subsection{Detailed Construction}
\label{subsec:BO:Construction}

{We briefly recall standard results on Gaussian Process regression and
introduce look-ahead acquisition functions.}

\subsubsection{Gaussian process regression}
We use $f$ to denote the GP surrogate of the objective $g$, which indicates $f(x^{(i)}) = y^{(i)}$ at observed point $x^{(i)}$. Initially, in BO we construct a GP prior and then update the posterior sequentially with new observations using GP regression, as we now describe. 

Assuming the data are collected randomly from a multivariate normal prior distribution with a specific mean and covariance matrix, we have
\begin{equation*}
    f(x^{(1:n)}) \sim N\left(\mu_0(x^{(1:n)}), \Sigma_0(x^{(1:n)},x^{(1:n)})\right).
\end{equation*}
Assuming further that $y^{(i)}$ is observed without noise (which can be easily relaxed), we can derive the posterior distribution as
\begin{align*}
    f(x)|\mathcal{D}_n \sim N(\mu_n(x), \sigma_n^2(x)) \, ,
\end{align*}
where
\begin{align*}
    \mu_n(x) &= \mu_0(x) 
    + \Sigma_0^{(1:n)}(x)(\Sigma_0^{(1:n)})^{-1}
    ( f(x^{(1:n)}) - \mu_0(x^{(1:n)}) ) \, ,\\
    \sigma_n^2(x) &=  \Sigma_0(x,x)
    - \Sigma_0^{(1:n)}(x)(\Sigma_0^{(1:n)})^{-1} \Sigma_0^{(1:n)}(x)^T \, ,
\end{align*}
with 
$\Sigma_0^{(1:n)}(x)=\Sigma_0(x,x^{(1:n)})$, 
$\Sigma_0^{(1:n)}=\Sigma_0^{(1:n)}(x^{(1:n)})$.

\subsubsection{Acquisition Functions}
In this work, we focus on multi-step acquisition functions which are formulated 
in terms of the underlying Markov decision process (MDP), 
defined as follows \cite{ginsbourger2010towards,lam2016bayesian,jiang2020efficient,astudillo2021multi,garnett2023bayesian}. 
The belief state of the MDP is the posterior $f_n(\cdot) = f(\cdot; \mathcal{D}_n) \in \mathcal{S}$, 
parameterized by mean $\mu_n(\cdot)$ and covariance kernel $\Sigma_n(\cdot, \cdot)$, 
where $\mathcal{S}$ is the space of state. 
The action is where we take observations $x \in \mathcal{X}$, 
and the stage-wise reward $r(f,x)$
characterizes the acquisition function:
\begin{equation*}	
    \alpha(x; \mathcal D_n) := 
    \mathbb{E}[ r(f,x) | \mathcal D_n ] \, . 
\end{equation*}
A simple reward function may look like
$r(f_n(x))$, in which case 
$\mathbb{E}[ r(f,x) | \mathcal D_n ]=\mathbb E[ r(f_n(x)) ]$.
{The input action of the MDP and its output updates the GP with a new observation} $(x^{(n+1)},y^{(n+1)})$
as described above, i.e. the dynamics of the MDP are given by
\begin{eqnarray*}
\mathcal F : \mathcal S \times \mathcal X \times \mathcal Y &\rightarrow& \mathcal S \, \\
(f_n, x^{(n+1)},y^{(n+1)}) &\mapsto& f_{n+1}(\cdot) = f(\cdot ; \mathcal D_{n+1}) \, .
\end{eqnarray*}
The BO algorithm provides a greedy stage-wise solution to the MDP by
repeating the following steps until the computational resources are exhausted: 
1) starting from an initial state (a GP $f_{n-1}$ with initial observations $\mathcal{D}_{n-1}$); 
2) 
{maximize the expected multi-step reward to determine the action $x_n$ at which we should take an observation;} 
3) observe the environment, $y_n$; 4) update the state with the new observation $f_{n-1} \mapsto f_n$ (see Algorithm \ref{alg:bo}).
Letting $\mathcal{D}$ 
denote all observations at the current state,
the multi-step look-ahead acquisition functions are
\begin{align}
    \alpha_0(x; \mathcal{D}) &:= 
 \mathbb E_{f(\cdot ; \mathcal D)} [ r(f,x) ] \notag \\
   \alpha_1(x; \mathcal{D}) &:=
   \mathbb E_{f(\cdot ; \mathcal D)} \left[ r(f,x) 
+ 
 \max_{x_1} \mathbb E_{f(\cdot ; 
\mathcal D_1(x))} 
 \left[ r( f,x_1 )  \right] \right]\notag\\
    \alpha_2(x; \mathcal{D}) &:=
  \mathbb E_{f(\cdot ; \mathcal D)} \Bigg[r(f,x) +  
\max_{x_1}  \mathbb E_{f(\cdot ; 
\mathcal D_1(x))} 
 \Big[ r( f,x_1)+
 \max_{x_2} \mathbb E_{f(\cdot;\mathcal D_2(x,x_1))} 
 [ r(f,x_2) ] 
   \Big ]
  \Bigg ] \,  \notag\\
   &\vdots \label{eq:alpha_1}
\end{align}
where
$\mathcal D_1(x) = \mathcal D \cup \{ (x , f(x;\mathcal D) )\}\, , 
\mathcal D_2(x,x_1) = \mathcal D_1(x) \cup \{ (x_1 , f(x_1;\mathcal D_1(x))) \}, \dots$,  
and we use $\mathbb E_{f(\cdot;\mathcal D)}$ to denote that the expectations above are taken over the Gaussian process $f$ 
given data $\mathcal D$.  
Here, the $\alpha_0$ is the standard one-step look-ahead acquisition function
and $\alpha_k$ denotes the $(k+1)$-step look-ahead acquisition function. 

\begin{algorithm}[tb]
   \caption{Bayesian optimization}
   \label{alg:bo}
\begin{algorithmic}
   \STATE {\bf Inputs}: $\mathcal{D}_0$, $\epsilon$, Black Box objective ``$g$", stopping criterion.
   \STATE {\bf Outputs}: $\hat{x}^* \approx {\sf argmin}_x g(x)$.
\STATE $n=0$.
   \WHILE{stopping criterion is not met}
    \STATE Compute Single Design (Alg. \ref{alg:single}) $\mathcal{D}_n \mapsto x_{n+1}$;
\STATE Evaluate Black Box $y_{n+1} = g(x_{n+1})$;
\STATE Augment $\mathcal{D}_{n+1} = \mathcal{D}_n \cup$ $\{x_{n+1},y_{n+1}\}$;
\STATE $n \mapsto n+1$.
    \ENDWHILE
    \STATE $\hat{x}^* = {\sf argmax}_{\mathcal{D}_n} g(x)$.
\end{algorithmic}
\end{algorithm}

Many reward functions lead to analytically intractable acquisition functions, 
so the 2-step look-ahead
acquisition function $\alpha_1$ requires 
nested MC approximation as follows 
\begin{align}\label{eq:OneStepNestedMC}
     \alpha_{1,N,M}(x;\mathcal{D}) &= \frac{1}{N}\sum_{i=1}^{N}
    \Bigg[r(f^i(x;\mathcal{D}))
+ \bigg(\max_{x_1^i}\frac{1}{M}\sum_{j=1}^{M}
    r(f^{ij}(x_1^i;\mathcal{D}_1^i(x)))
    \bigg)\Bigg].
\end{align}
Additional details are provided in Appendix \ref{app:acquisition}.
In general, $n$-step look-ahead acquisition functions
can be approximated with $n$-nested MC approximations.
{In the present work, we will focus on deriving theories and showing numerical results for
{\bf single-nested MC}, which is amenable to MLMC.} 
This includes 2-step look-ahead q-EI \cite{wang2020parallelBO} and 3-step look-ahead EI.
{As we will see, the 2-step look ahead formulation is sufficient to gain substantial benefits in terms of MSE when compared with SOTA methods, which are either single MC, e.g. 1-step qEI or 2-step EI \cite{wu2019practical,ginsbourger2010towards,chevalier2013fast,renganathan2020recursive,balandat2020botorch},
or rely on very rough approximations for more than 2 steps \cite{lam2016bayesian,jiang2020efficient,astudillo2021multi,garnett2023bayesian}.}
There is an extension of the MLMC technology that is applicable 
to the general $n$-nested case, which is called multi-index MC 
\cite{haji2016multi,ourmimc,jasra2023multi}, and that will be the topic of future 
investigation.

Numerous reward functions can be used in BO (see e.g., \cite{frazier2018tutorial,balandat2020botorch}),
but for the present exposition, we focus on two of the
most ubiquitous ones. 

\textbf{Expected improvement (EI)} \cite{movckus1975bayesian} is obtained by setting 
$$
{r(f,x) := (f(x; \mathcal D) - f^*(\mathcal{D}))_+ ,}$$ with $(\cdot)_+ = \max\{0, \cdot\}$ and $f^*(\mathcal{D}) = \max_{(x,f(x))\in \mathcal{D}} f(x)$. 
Expected improvement can be computed analytically as
\begin{align*}
    {\sf EI}(x | \mathcal{D}) &:= \mathbb{E}\left[(f(x; \mathcal{D}) - f^*(\mathcal{D}))_+|\mathcal{D}\right]\\
    &= \sigma(x) z(x)\Phi\left(z(x)\right) + 
    \sigma(x)\phi\left(z(x)\right)
\end{align*}
where $z(x) := {(\mu(x) - f^*(\mathcal{D}))}/{\sigma(x)}$
and 
where $\phi(\cdot)$ and $\Phi(\cdot)$ are the PDF and CDF of the standard normal distribution, respectively.

\noindent
\textbf{q-Expected improvement (qEI)} \cite{wang2020parallelBO} maximizes over a batch $x=(x_1,\dots,x_q)$ of $q>1$ points 
jointly, i.e., $r(f, x) = \max_{j=1,...,q}(f(x_j; \mathcal{D}) - f^*(\mathcal{D}))_+$ so that 
\begin{equation}\label{eq:qEI}
    {\sf qEI}(x| \mathcal{D}) = \mathbb{E}\left[\max_{j=1,...,q}(f(x_j; \mathcal{D}) - f^*(\mathcal{D}))_+|\mathcal{D}\right].
\end{equation}
This is analytically intractable but can be approximated by MC estimation
\begin{equation*}
    {\sf qEI}(x | \mathcal{D}) \approx 
    \frac{1}{N}\sum_{i=1}^{N}\max_{j=1,...,q}(f^i(x_j;\mathcal{D}) - f^*(\mathcal{D}))_+.
\end{equation*}

Examples of two-step look-ahead EI, two-step look-ahead qEI, and three-step look-ahead EI are explicitly given in Appendix \ref{sec:Exp12EI}. Noting that the latter two both
result in 
a single-nested MC objective, we henceforth 
denote it by $\alpha(x;\mathcal{D})$ without subscripts,
and the corresponding optimizers by $x^*$.

\section{Sample Average Approximation}
\label{sec:saa}

{ 
Sample average approximation (SAA)} \cite{kleywegt2002sampleaverage,balandat2020botorch} 
is constructed from i.i.d. samples as
\begin{align}\label{eq:SAA}
    x_N^{} &:= \argmax_{x\in\mathcal{X}}  
    \frac{1}{N}\sum_{i=1}^{N}r(f^i(x; \mathcal{D}_{})) \\ \label{eq:original}
    &\approx \argmax_{x\in\mathcal{X}}\ \alpha(x; \mathcal{D}_{}) =: x^{*} \, .
\end{align}

Below and hereafter, $\|\cdot\|$ denotes the Euclidean distance. 
Proposition~\ref{thm:SAAoptimiser} below guarantees the rate of convergence of the maximizer under
standard assumptions: essentially, 
$x\mapsto\alpha(x;\mathcal D)$ has a compact domain with Lipschitz derivatives and is locally quadratic at its optimum (see Appendix~\ref{app:saa} for a more formal statement and further discussion). 

\begin{prop}[Theorem 12 of \cite{kim2015guide}]\label{thm:SAAoptimiser}
Given a unique optimizer and Assumption \ref{assumptions} 
$$\mathbb{E}[\|x_N - x^*\|] = O (N^{-1/2}) \, .$$
\end{prop}
In practice, we assume that our optimization algorithm can find the global optimizer,
and we attempt to achieve this with a committee of 
multiple initializations of a local optimization algorithm. 
For given realizations, 
the function $\alpha_N(\cdot)$ is deterministic, 
and we can apply deterministic optimization algorithms such as L-BFGS with multi-start initialization. 
We now show the decomposition of approximation error of AF \eqref{eq:OneStepNestedMC}, 
omitting the subscript 1 henceforth.

\begin{prop}\label{prop:MCrate}
If Assumption \ref{assumptions} holds and 
$\forall x, x_1 \in \mathcal{X},  \mathsf{Var}\left(
    r(f,x) + 
 \mathbb E_{f(\cdot ; 
\mathcal D_1(x))} 
 \left[ r( f,x_1 )  \right]\right) \leq \sigma_{\max}$ and $\mathsf{Var}(r( f,x_1 ))\leq \sigma_{\max},$ for some $\sigma_{\max} > 0$,  
then the following holds, for some $C>0$
\begin{align}\label{eq:nonas}
\sup_{x\in \mathcal{X}} \bbE \bigl[| \alpha_{N,M}(x;\mathcal{D}) - \alpha(x;\mathcal{D})|^2 \bigr] \leq 
C\Big(\frac{1}{N} 
+ \frac{1}{M}\Big)  \, .
\end{align}
\end{prop}
The proof is provided in Appendix~\ref{sec:proof1}. 
The decomposition for three-step look-ahead EI \eqref{eq:TwoStepMC} follows 
similarly.
The rate with respect to the outer and inner MC is numerically verified in Appendix \ref{sec:NumMCRates}.

Given the decomposition of MSE as $O(1/N+1/M)$, if
we want to achieve an MSE of $\varepsilon^2$ for $\varepsilon > 0$, 
then we require $N = O(\varepsilon^{-2})$ and $M = O(\varepsilon^{-2})$, 
hence a computational complexity of $MN = O(\varepsilon^{-4})$. 
If there is no nested MC, such as $\alpha_0$ from \eqref{eq:alpha_1}, 
{the computational complexity} to achieve an MSE of $\varepsilon^2$ is $O(\varepsilon^{-2})$, 
which is called the canonical rate of MC.
We can leverage MLMC to reduce 
{the complexity of $\alpha_1$ 
from sub-canonical} $O(\varepsilon^{-4})$ if \eqref{eq:OneStepNestedMC} is used, 
to canonical $O(\varepsilon^{-2})$.

\section{Multilevel Monte Carlo}
\label{sec:mlmc}

The idea behind MLMC is relatively simple \cite{giles2015multilevel}. We first introduce the general concept and then discuss the specific construction for BO. 
Suppose we want to estimate the expectation of some quantity of interest $\varphi$, which must itself be approximated.
{Here the quantity of interest is the given look-ahead stage-wise reward and the approximation refers to the inner MC.}
We leverage a sequence of approximations to $\varphi$ with increasing accuracy and cost,
where an initial estimate using
a large number of samples 
with low accuracy and cost is 
corrected 
with progressively fewer-sample estimates
of {\em increments} with higher accuracy and cost.
We make this precise in the following.

Denote the sequence of approximations with increasing accuracy and cost over levels $l$ 
by $\varphi_{M_0}, \varphi_{M_1},\dotsc,\varphi_{M_L}$. 
Let $M_l = 2^l$. 
The MLMC approximation to $\bbE[\varphi]$ follows from the telescoping sum identity
\begin{align*}
    \bbE[\varphi] \approx \bbE[\varphi_{M_L}] 
    &= \sum_{l=0}^{L}\bbE[\Delta\varphi_{M_l}],
\end{align*}
where $\Delta\varphi_{M_l} = \varphi_{M_l} - \varphi_{M_{l-1}}$ is the increment and $\varphi_{M_{-1}}~=~0$. 
Now we let $$Z_{N_l, M_l} = \frac{1}{N_l}\sum_{i=1}^{N_l}(\varphi_{M_l}^i - \varphi_{M_{l-1}}^i)$$ be an unbiased MC estimator of $\bbE[\Delta\varphi_{M_l}]$ for $l = 0,1,...,L$,
where $\varphi_{M_l}^i$ and $\varphi_{M_{l-1}}^i$ are computed using the same 
input samples, 
i.i.d for each $i$ and $l$.
If $|\bbE[\varphi_{M_l} - \varphi]|$ and the variance $\mathsf{Var}(\Delta\varphi_{M_l})$ decay exponentially as $l$ increases with specific rates, we can achieve computational benefits over the standard MC method. The following theorem is the standard MLMC theorem similar to \cite{giles2015multilevel}.
\begin{prop}\label{prop:mlmc}
Assume there exists positive constants $s, w, \gamma$ with $s\geq\frac12\min\{\beta, \gamma\}$ and $C$ (where $C$ may vary line to line) such that 
\begin{itemize}
    \item $B_{l} := |\bbE[\varphi_{M_l} - \varphi]| \leq C M_l^{-s} = C2^{-sl}$
    \item $V_{l} := \mathsf{Var}(\Delta\varphi_{M_l}) \leq CM_l^{-\beta} = C2^{-\beta l}$
    \item ${\sf COST}_{l} := {\sf COST}(\Delta\varphi_{M_l}) \leq  CM_l^{\gamma} = C2^{\gamma l}$,
\end{itemize}
then there are values $L$ and $N_l$ such that the multilevel estimator 
\begin{equation*}
    Z = \sum_{l=0}^{L}Z_{N_l,M_l}
\end{equation*}
can be approximated with an accuracy $\varepsilon^2$ measured by the mean square error (MSE) with a computational complexity ({\sf COST}) for which \begin{equation*}
    \bbE[{\sf COST}] \leq 
    \begin{cases}
    C\varepsilon^{-2},&\beta>\gamma,\\
    C\varepsilon^{-2}(\log\varepsilon)^2,&\beta=\gamma,\\
    C\varepsilon^{-2-(\gamma-\beta)/s},&\beta<\gamma.\\
    \end{cases}
\end{equation*}
\end{prop}

MLMC requires a regularity assumption on bias and incremental variance convergence. For the canonical situation, we require the variance to decay faster than the cost increase such that the cost at the lower level dominates; for the borderline situation, the cost is equally spread among levels; for the worst situation, the cost at the finer level dominates that corresponds to the MC rate. 
{A detailed proof can be found in \cite{giles2008multilevel}. 
The main components 
are provided in Lemma~\ref{lem:costmin} in Appendix~\ref{sec:proof2}.}

\subsection{Multilevel formulation of the maximizer of the acquisition function} \label{policy_version}

We construct a multilevel estimator for the two-step look-ahead acquisition function
\begin{equation}\label{eq:Onestep}
    {\displaystyle \alpha(x; \mathcal{D}) :=
   \mathbb E_{f(\cdot ; \mathcal D)} \left[ r(f,x) 
+ 
 \max_{x_1} \mathbb E_{f(\cdot ; 
\mathcal D_1(x))} 
 \left[ r( f,x_1 )  \right] \right].}
\end{equation}
Since our target is to find the next observation point guided by the acquisition function, we formulate the multilevel estimator for the maximizer as
\begin{equation}\label{eq:mlmcest}
x^*_{\sf ML} := \sum_{l=0}^L z_{N_l,M_l} -z_{N_l,M_{l-1}}=\sum_{l=0}^L\Delta z_{N_l,M_l}\, ,
\end{equation}
where 
\begin{align}
    z_{N_l,M_l}& = \argmax_{x\in \calX}
    \frac{1}{N_l}\sum_{i=1}^{N_l}
    \Bigg[r(f^i(x;\mathcal{D}))
+ \bigg(\max_{x_1^{i}}\frac{1}{M_l}\sum_{j=1}^{M_l}
    r(f^{ij}(x_1^{i};\mathcal{D}_1^i(x)))
    \bigg)\Bigg],\label{eq:MLfine}\\
    z_{N_l,M_{l-1}}& = \argmax_{x\in \calX} \frac{1}{N_l}\sum_{i=1}^{N_l}
    \Bigg[r(f^i(x;\mathcal{D}))
+    \bigg(\max_{x_1^{i}}\frac{1}{M_{l-1}}\sum_{j=1}^{M_{l-1}}
    r(f^{ij}(x_1^{i};\mathcal{D}_1^i(x))) 
    \bigg)\Bigg] \, ,\label{eq:MLcoarse}
\end{align}
where $z_{N_0,M_{-1}} \equiv 0$. To leverage MLMC, we first simulate $M_l$ samples of $f(x_1;\mathcal{D}_1(x))$ and construct the $z_{N_l,M_l}$ and then subsample $M_{l-1}$ samples from the $M_l$ samples to construct $z_{N_l,M_{l-1}}$.
We remark that the terms $\Delta z_{N_l,M_l}(x)$ for $l=0,1,...,L$ are mutually independent. 
{Algorithm~\ref{alg:MLMC} shows a general implementation of the MLMC approximation, 
where $\Phi_l^{l}$ and $\Phi_l^{l-1}$ denote the fine and coarse approximation of AF at level $l$
and $\Phi_0$ is the single AF at level~$0$.}

\begin{algorithm}[htb]
   \caption{Single Design with MLMC acquisition}\label{alg:MLMC}
   \label{alg:single}
\begin{algorithmic}
   \STATE {\bf Inputs}: $\mathcal{D}_n$, $\epsilon$
   \STATE {\bf Outputs}: $x_{n+1}$
\STATE {\bf Define} $L$, $N_0,\dots,N_L$, and $M_0, \dots, M_L$ (as in Thm \ref{thm:qFunc})
\STATE {\bf Compute}
$z_0 = {\sf argmin}_z \Phi_0(z)$
   \FOR{$l=1,\dots, L$}
    \STATE $z_l^l = {\sf argmin}_z \Phi_l^l(z)$
and $z_l^{l-1}= {\sf argmin}_z \Phi_l^{l-1}(z)$;
\STATE $\Delta_l = z_l^l - z_l^{l-1}$.
    \ENDFOR
    \STATE $x_{n+1} = z_0 + \sum_{l=1}^L \Delta_l$ .
\end{algorithmic}
\end{algorithm}

An {\em antithetic coupling} approach can improve the rate of convergence.
It involves replacing
the coarse estimator \eqref{eq:MLcoarse} with
\begin{align}
    z_{N_l,M_{l-1}}^A &= \argmax_x\frac{1}{N_l}\sum_{i=1}^{N_l}
    \Bigg[r(f^i(x;\mathcal{D}))\nonumber\\
    &\quad + \frac12\bigg(\max_{x_1^{i}}\frac{1}{M_{l-1}}\sum_{j=1}^{M_{l-1}}r(f^{ij}(x_1^{i};\mathcal{D}_1^i(x))) 
+    \max_{x_1^{i}}\frac{1}{M_{l-1}}\sum_{j=M_{l-1}+1}^{M_l}
    r(f^{ij}(x_1^{i} ;\mathcal{D}_1^i(x))) 
    \bigg)\Bigg] \, ,
    \label{eq:Antcoarse}
\end{align}
where $z^{A}_{N_0,M_{-1}} \equiv 0$ and the samples used for $z_{N_l,M_{l-1}}^A$ are taken by splitting the $M_l$ samples from $z_{N_l,M_l}$. In the coarser estimator $z_{N_l,M_{l-1}}^A$, one averages the maxima associated with two batches of $M_{l-1}$ samples rather than maximizing the average. 
It can easily be shown using Taylor expansion for smooth
functions \cite{blanchet} that this can double the rate, 
but it can also yield improvement even for non-smooth functions \cite{giles2018mlmc, giles2019decision}, as in our situation.
The multilevel formulation for the three-step look-ahead AF is given in Appendix \ref{sec:ML2Step}. 
%

{
The complexity is determined by 
the cost to evaluate the terms above, which is
$O(d(\sum_{l=0}^{L}N_l(M_l+1))$ 
for both the antithetic and regular cases.
Without loss of generality, we assume that a fast super-linear solver is available, 
so computing \eqref{eq:MLfine}
can be done with a 
linear cost in the evaluation of the objective function, 
{which corresponds to $\gamma = 1$}.

Note the maximizer increments of \eqref{eq:mlmcest} 
are {\em not unbiased},
as is common in Bayesian MLMC \cite{beskos2017multilevel}.
In the present work, we provide theory for a simplified estimator, 
and introduce strong assumptions for \eqref{eq:mlmcest}.
The more subtle and challenging complete theory is deferred to future work.

\subsection{MLMC approximation of the acquisition function}\label{sec:qfunc_approach}

We first present a theoretical analysis of a procedure for MLMC BO using the easier-to-analyse AF 
results. We combine MLMC ideas with SAA techniques on AF 
instead of 
optimizers, i.e. we directly construct an MLMC estimator for the AF.

{Similarly to \eqref{eq:MLfine} and \eqref{eq:MLcoarse}, we construct $\Delta \alpha_{N_l,M_l}(x) = \alpha_{N_l,M_l}(x) -\alpha_{N_l,M_{l-1}}(x)$, with
$\alpha_{N_0,M_{-1}}(x) \equiv 0$, as incremental approximations of \eqref{eq:Onestep}.
From this, we form the estimator
\begin{equation*}
\alpha_{\sf ML}(x) := \sum_{l=0}^L\Delta \alpha_{N_l,M_l}(x)\, .
\end{equation*}
We now argue that $\alpha_{\sf ML}$ is a computationally efficient proxy for the intractable limit objective $\alpha(x;\mathcal{D})$, which leads to computational benefits for approximating optimizer $x^*$ (see Appendix \ref{app:mlmc_af} for some further details).}

Before proceeding with the theoretical analysis, we note again that the above estimator $\alpha_{\sf ML}(x)$ and the corresponding $\argmax_{x\in\mathcal{X}} \alpha_{\sf ML}(x)$ is just one estimator that leverages the MLMC methodology.
It is particularly tractable from a theoretical perspective due to unbiased increments. 
We require an additional technical Assumption \ref{ass:IncVar}, which holds for GP with standard kernels.
This and supporting Lemmas are presented in Appendix \ref{app:mlmc}.

\begin{theorem}[Q-Function Convergence]\label{thm:qFunc}
    Suppose Assumptions~\ref{assumptions} and~\ref{ass:IncVar} hold. Let $M_l = 2^l$. For $x\in\mathcal{X}$, the MLMC estimator $\alpha_{\sf ML}(x)$ is such that 
    \begin{align*}
        \mathbb E\Big[\sup_x \;& |\alpha_{\sf ML}(x) - \alpha(x;\mathcal{D})|^2\Big]
        \leq \ 
        C(L+1)\left(\frac{V_0}{N_0} + \sum_{l=1}^L\frac{1}{N_lM_l}\right) + \frac{C}{M_L} \, ,
    \end{align*}
        \begin{align*}
        \mathbb E\Big[\sup_x \;& |\nabla\alpha_{\sf ML}(x) - \nabla\alpha(x;\mathcal{D})|^2\Big]
        \leq \ 
C(L+1)\left(\frac{V_0}{N_0} + \sum_{l=1}^L\frac{1}{N_lM_l}\right) + \frac{C}{M_L} \, ,
    \end{align*}
     for some constant $C$. Consequently taking\footnote{The ceiling function provides integer values of $L$, $N_l$, and $M_l$.} $L = O(|\log\varepsilon|),$
\begin{align}\label{eq:params}
    N_0 = \varepsilon^{-2}K_L\left(\frac{V_0}{M_0}\right)^{1/2}\, \ \text{and}\ \ N_l  =\varepsilon^{-2}K_L\frac{1}{M_l}, 
\end{align}
for $l \geq 1$ with $K_L = (V_0M_0)^{1/2} + L\, ,$
the MLMC estimator $\alpha_{\sf ML}(x)$ gives the following error estimate, for $x\in\mathcal{X}$,  $$\mathbb E\bigl[\sup_x |\alpha_{\sf ML}(x) - \alpha(x;\mathcal{D})|^2\bigr]\leq C\varepsilon^2$$ for a sample complexity of $O(\varepsilon^{-2}|\log\varepsilon|^3)$.
\end{theorem}
The proof of Theorem \ref{thm:qFunc} is in Appendix \ref{sec:ProofqFunc}.
Appendix \ref{app:value_max} provides corollaries for the value function and the maximizer. 
{Here, we consider the inner $\sup$ with a loose bound; up to logarithmic terms, we achieve the same rate as a standard MC stimulation (without nested summations and maximizations).}

\subsection{Main Result}\label{sec:main}

We now establish the main MLMC complexity result, which requires 
yet another additional Assumption \ref{ass:qImpOpt},
given in Appendix \ref{app:mlmc}.
This final assumption essentially states that the strong rate of 
convergence of increments of the acquisition function $\Delta \alpha_{N_l,M_l}$
implies the corresponding rate for the maximizer $\Delta z_{N_l,M_l}$.
As indicated earlier, we expect the result to hold for the maximizer, 
however the proof of this is deferred to future work.

{\begin{theorem}\label{thm:main}
Suppose Assumptions~\ref{assumptions}, \ref{ass:IncVar} and \ref{ass:qImpOpt} hold.  
Let $M_l = 2^l$.
    Then one can choose $L,\{N_l\}_{l=0}^L$, as in Theorem \ref{thm:qFunc},
    such that
\begin{equation*}
    \bbE [\| x^*_{\sf ML} - x^*\|^2] \leq C \varepsilon^2 \, ,
\end{equation*}
for a complexity of $O(\varepsilon^{-2}(\log\varepsilon)^2)$.
\end{theorem}}

The proof is given in Appendix \ref{sec:ProofOptimizer}. We verify $\beta=1$ numerically with the difference of optimizer in Appendix \ref{sec:NumMLMCRates}.
This corresponds to the borderline scenario of Proposition~\ref{prop:mlmc} with $\beta=\gamma=1$, which still yields an improvement compared with the standard MC. We have the following remark, which can further improve the computational complexity.

\begin{rem}\label{rem:Ant}
    The antithetic construction \eqref{eq:Antcoarse} gives $\beta \approx 1.5$ leading to an overall canonical computational complexity $O(\varepsilon^{-2})$. This is similar to the work by \cite{giles2019decision}, but here, we are interested in constructing a multilevel estimator of the optimizer. The rates are numerically verified with the 1D toy example \ref{eq:toy}. See Appendix \ref{sec:NumMLMCRates} for the verification of $\beta\approx1.5$ and Figure \ref{fig:1DMSE} for the complexity.
\end{rem}

\section{Numerical Results}
\label{sec:num}
\begin{figure}[htb]
    \centering
    \vspace{-0.2cm}
    \begin{subfigure}{0.495\linewidth}
    \includegraphics[width=1.18\linewidth]{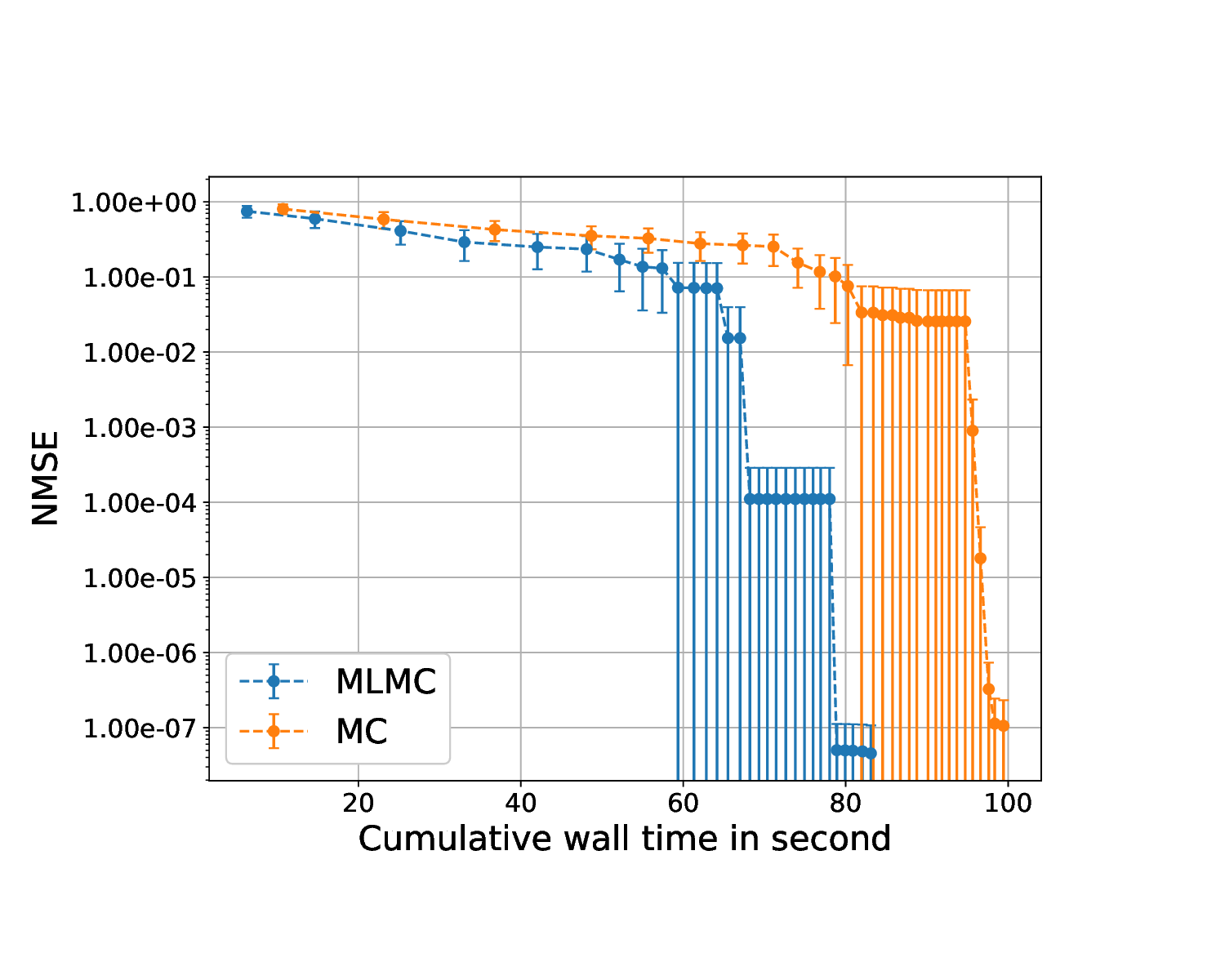}
    \vspace{-1.2cm}
    \captionsetup{width=1.18\linewidth}
    \caption{1D Toy Example (d=1)}
    \label{fig:1DToyBO}
    \end{subfigure}
    \begin{subfigure}{0.495\linewidth}
    \includegraphics[width=1.18\linewidth]{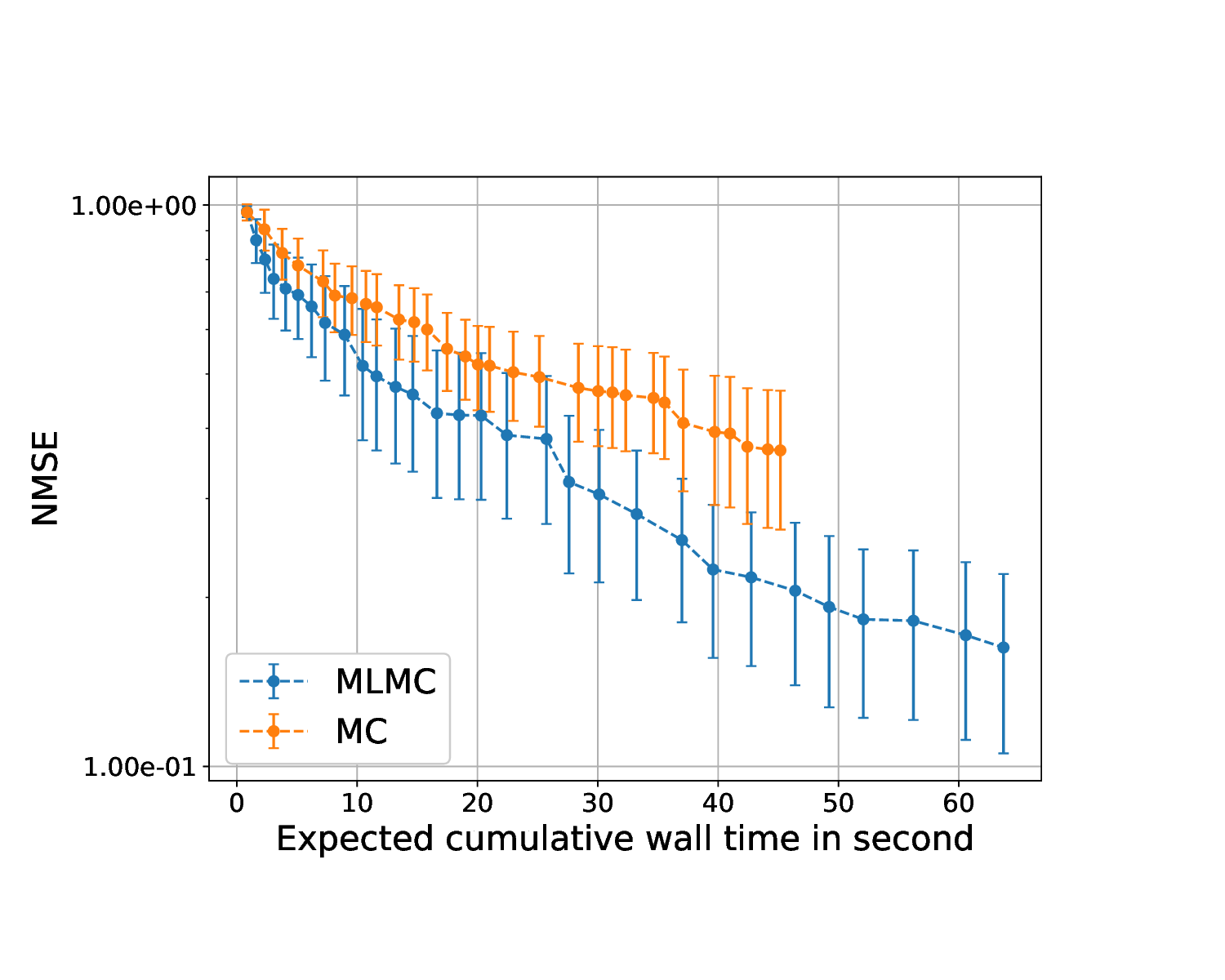}
    \vspace{-1.2cm}
    \captionsetup{width=1.18\linewidth}
    \caption{Ackley (d=2)}
    \label{fig:Ackley2}
    \end{subfigure}

    \begin{subfigure}{0.495\linewidth}
    \includegraphics[width=1.18\linewidth]{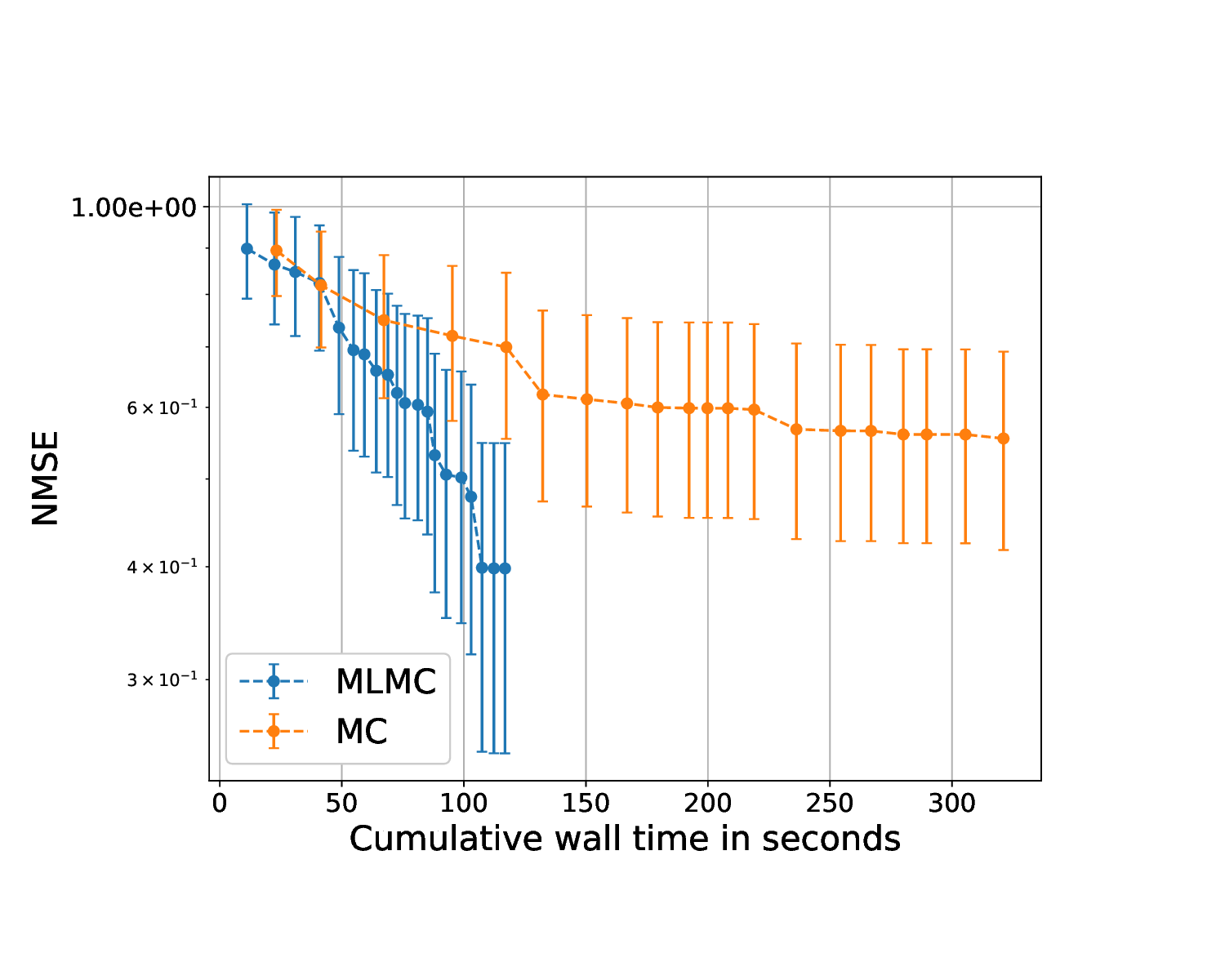}
    \vspace{-1.2cm}
    \captionsetup{width=1.18\linewidth}
    \caption{DropWave (d=2)}
    \label{fig:DropWave}
    \end{subfigure}
    \begin{subfigure}{0.495\linewidth}
    \includegraphics[width=1.18\linewidth]{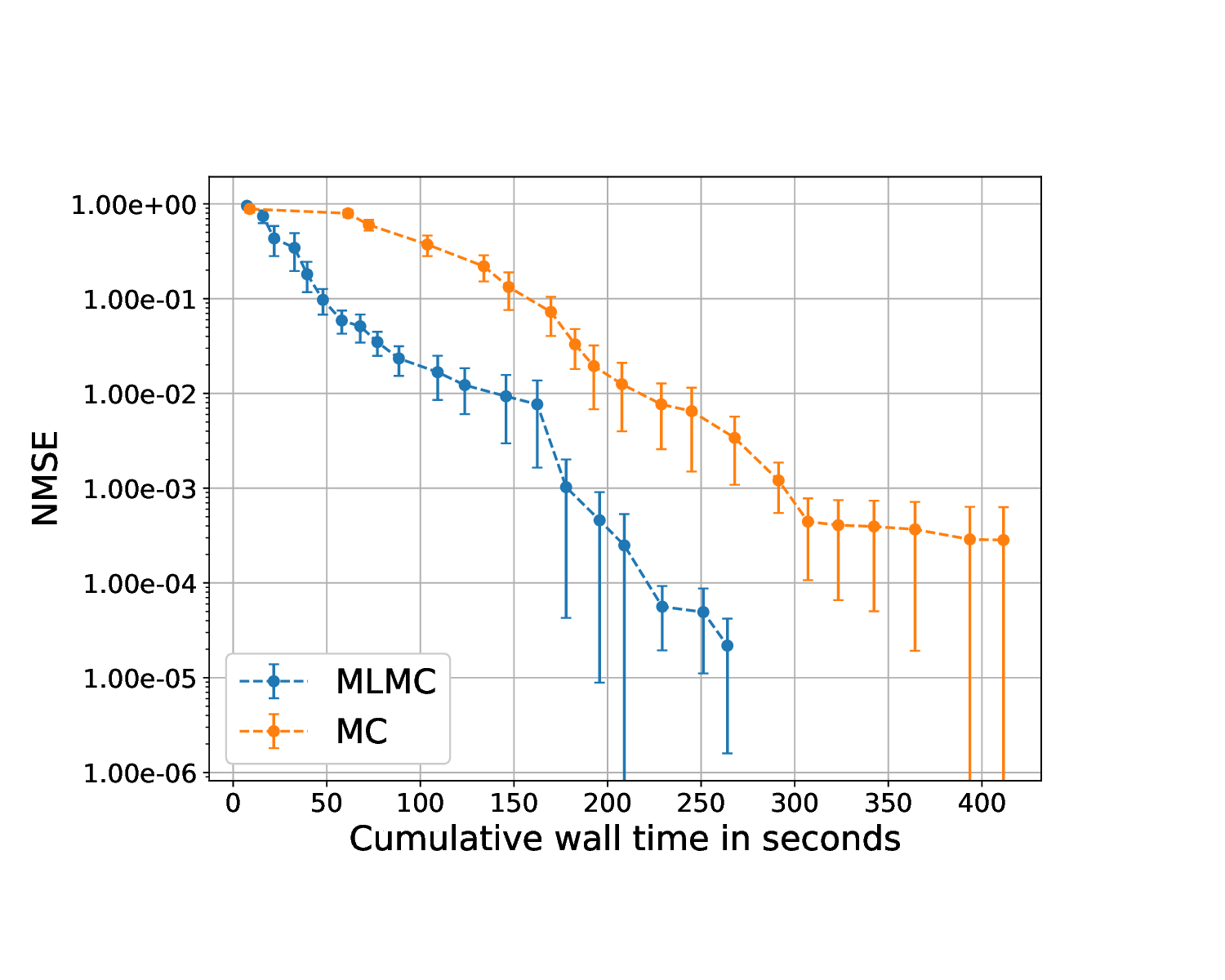}
    \vspace{-1.2cm}
    \captionsetup{width=1.18\linewidth}
    \caption{Branin (d=2)}
    \label{fig:Branin}
    \end{subfigure}

    \begin{subfigure}{0.495\linewidth}
    \includegraphics[width=1.18\linewidth]{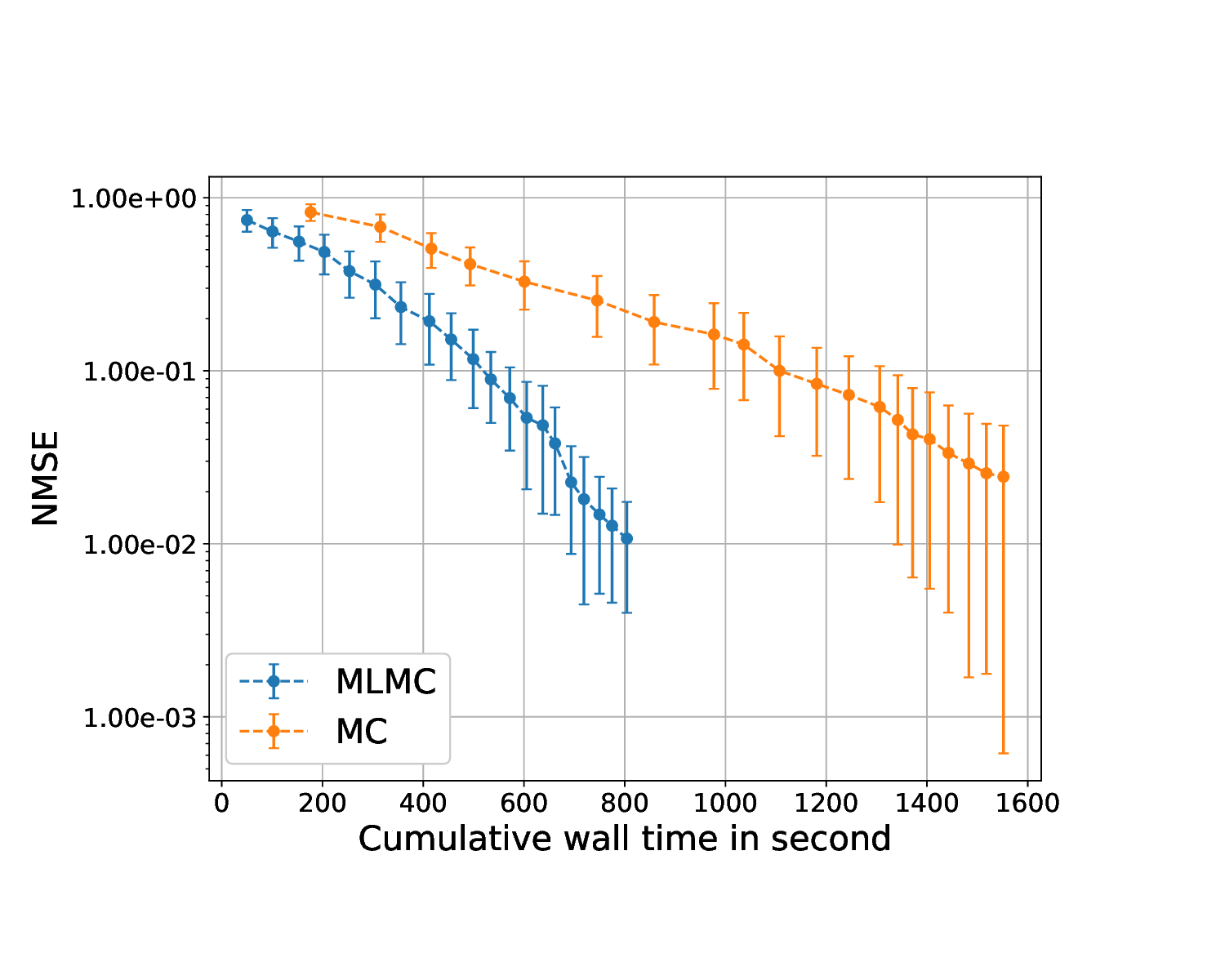}
    \vspace{-1.2cm}
    \captionsetup{width=1.18\linewidth}
    \caption{Hartmann6 (d=6)}
    \label{fig:Hartmann}
    \vspace{-0.3cm}
    \end{subfigure}
    \begin{subfigure}{0.495\linewidth}
    \includegraphics[width=1.18\linewidth]{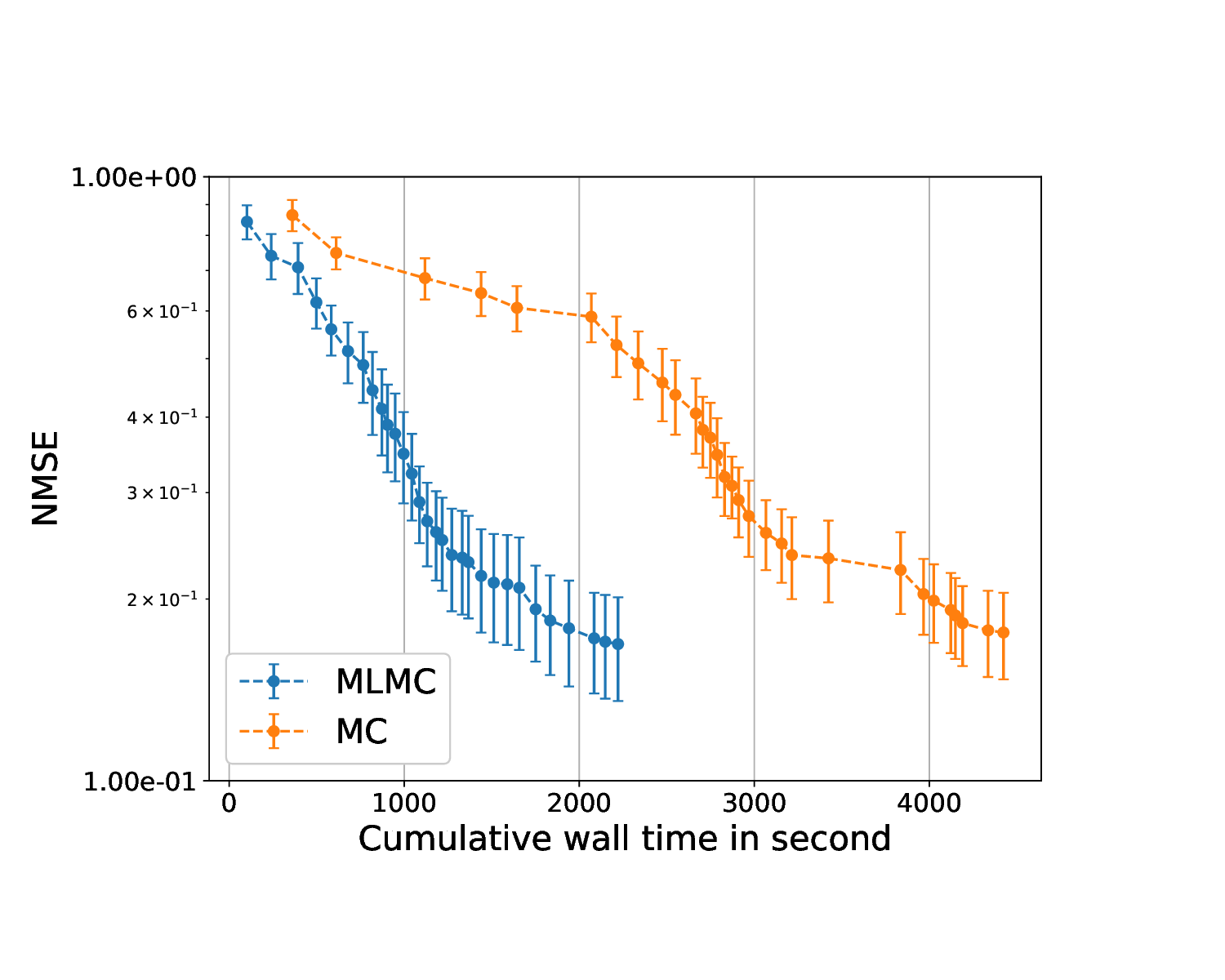}
    \vspace{-1.2cm}
    \captionsetup{width=1.18\linewidth}
    \caption{Cosine8 (d=8)}
    \label{fig:Cosine}
    \vspace{-0.3cm}
    \end{subfigure}
    \caption{Convergence of the BO algorithm with respect to the cumulative wall time in seconds, with error bars (computed with 20 realizations). The Mat{\'e}rn kernel is applied. The initial BO run starts with $2\times d$ observations.}
    \label{fig:fullBOMLqEI}
\end{figure} 
\begin{table*}[t!]
\centering
\resizebox{\textwidth}{!}{\begin{tabular}{c c c c c c c c c c c c c } 
\hline
Function Name & & PI & EI & UCB & PES & GLASSES & R-4-9 & R-4-10 & R-5-9 & R-5-10 & Ours \\
\hline
\multirow{2}{7em}{Branin-Hoo} & \multicolumn{1}{l}{Mean} & 0.847 & 0.818 & 0.848 & 0.861 & 0.846 & 0.904 & 0.898 & 0.887 & 0.903 & \textbf{0.992}\\
& \multicolumn{1}{l}{Median} & 0.922 & 0.909 & 0.910 & 0.983 & 0.909 & 0.959 & 0.943 & 0.921 & 0.950 & \textbf{0.997}\\
\hline

\multirow{2}{7em}{Goldstein-Price} & \multicolumn{1}{l}{Mean} & 0.873 & 0.866 & 0.733 & 0.819 & 0.782 & 0.895 & 0.784 & 0.861 & 0.743 & \textbf{0.987}\\ 
& \multicolumn{1}{l}{Median} & 0.983 & 0.981 & 0.899 & 0.987 & 0.919 & 0.991 & 0.985 & 0.989 & 0.928 & \textbf{0.998}\\ 
\hline

\multirow{2}{7em}{Griewank} & \multicolumn{1}{l}{Mean} & 0.827 & 0.884 & 0.913 & 0.972 & 1.0\footnotemark{}
& 0.882 & 0.885 & 0.930 & 0.867 & \textbf{0.973}\\
& \multicolumn{1}{l}{Median} & 0.904 & 0.953 & 0.970 & \textbf{0.987} & 1.0\footnotemark[\value{footnote}]
& 0.967 & 0.962 & 0.960 & 0.954 & 0.975\\
\hline

\multirow{2}{7em}{6-hump Camel} & \multicolumn{1}{l}{Mean} & 0.850 & 0.887 & 0.817 & 0.664 & 0.776 & 0.860 & 0.825 & 0.793 & 0.803 & \textbf{0.942}\\
& \multicolumn{1}{l}{Median} & 0.893 & 0.970 & 0.915 & 0.801 & 0.941 & 0.926 & 0.900 & 0.941 & 0.907 & \textbf{0.979}\\
\hline
\end{tabular}}
\caption{[Reproduced with permission from \cite{lam2016bayesian}]. 
Performance of our 2-step look-ahead acquisition function (1EI + 2EI) on benchmark functions. 
The rollout algorithms of \cite{lam2016bayesian} are denoted 
R-.... The other algorithms are commonly used and all are myopic except GLASSES. 
The error metric is GAP $= (g(x_0) - \max_{\mathcal{D}_k}g(x))/(g(x_0) - g(x^*))$ after a fixed budget of 
$k=15$ function evaluations, 
and statistics are from $40$ realizations randomly initialized with a single function evaluation.} 
\label{tab:ComparisonLam}
\end{table*}

{In this section, we run the whole BO algorithm on various synthetic functions from BoTorch \cite{balandat2020botorch} and show the benefits of using MLMC \eqref{eq:mlmcest} with the antithetic approach over standard MC \eqref{eq:OneStepNestedMC} in AF approximation. We apply the normalized MSE, defined as
\begin{equation}\label{eq:NMSE}
    \text{NMSE} = \mathbb E\left[\frac{\|\max_{\mathcal{D}_k}g(x)-g(x^*)\|^2}{\|\max_{\mathcal{D}_0}g(x)-g(x^*)\|^2}\right],
\end{equation}
where $\mathcal{D}_0$ and $\mathcal{D}_k$ are the initial observation set and the observation set after $k$ iterations, as the error metric. It measures the algorithm's expected relative improvement over the initial values. It normalizes the value to $[0,1]$, which can make the performance on different problems with different scales of domain comparable and remains the decaying feature of error.  
The AF we applied is the two-step look-ahead 1-EI + 2-EI defined in \eqref{eq:1EIqEI}. According to Figure \ref{fig:fullBOMLqEI}, in expectation, compared with the MC method, the MLMC method can achieve the same NMSE with less computational cost to find the optimum of test functions. 
{We primarily focus on the method from Section \ref{policy_version} as our initial study. {Comparisons between the multilevel optimizer method from Section \ref{policy_version} and the multilevel function method from Section \ref{sec:qfunc_approach} are performed in Section \ref{app:thm1}. }

{The number of samples in MC is $N = M = 1/\varepsilon^2$. The number of samples in MLMC for $N_l$ and $M_l$, related to $\varepsilon$, is chosen according to the formula in Theorem \ref{thm:qFunc}. For comparison, the $\varepsilon$ is set to be the same for any given problem. Here, we use $\varepsilon=0.2$ for Figure \ref{fig:1DToyBO}, \ref{fig:Ackley2}, \ref{fig:DropWave}, \ref{fig:Branin} and \ref{fig:Cosine}, and $\varepsilon=0.15$ for Figure \ref{fig:Hartmann}.} It is useful to note that MLMC requires balancing the "variance" and "bias" to achieve the canonical rate. This requires careful tuning in selecting the finest level and the number of samples at each level. From the previously derived formulas, the level and the number of samples depend on the variance of the increments. In practice, since we need to approximate the AF interactively, it is not realistic nor efficient to compute the variance every time before we evaluate the functions. A practical way to deal with that is to use $C2^{-\beta l}$, which is due to the proportionality of the variance, with the constant $C$ empirically. It turns out that the worst situation is that MLMC degenerates to MC, and we have an MC rate of convergence in a single AF approximation. Overall, MLMC with this pragmatic choice performs much better in practice.}

{
Because L-BFGS is a local optimization method,
we need to run it with multiple initial conditions to get an approximation of the global maximizer.
An additional complication in the present context is that we 
need to match the coarse and fine maximizers 
at each level $l$ in order for \eqref{eq:mlmcest} to deliver a good approximation.
The simplest way to achieve this is by matching all levels with the best one from level $0$ (or~$L$), 
which is the strategy adopted in the present work.
However, 
we want the estimator \eqref{eq:mlmcest} which delivers the maximum value, 
rather than being anchored to any individual level.
This can be achieved with backward (or forward) matching, 
 where we essentially prune a tree whose branches start 
 at level~0 (or~L) and branch forward (or backward) in levels.
This approach is more robust but also slightly more complicated and costly, 
and we have found that the former strategy is satisfactory in practice.
The approaches are illustrated in Figures \ref{fig:PointMatching} and \ref{fig:BackwardMatching} 
in Appendix~\ref{sec:Matching}.
A global strategy which explores all combinations of the local optimizers of the 
$2L$ intermediate problems would incur a non-trivial cost of $O(\varepsilon^{-2})$.
}

In Table~\ref{tab:ComparisonLam} we replicate Table 2 of \cite{lam2016bayesian} and include a comparison with our 2-step look-ahead AF (1EI + 2EI) computed using MLMC.  We observe that the 2-step look-ahead AF (1EI + 2EI) is the best in 7 out of 8 cases and the second best for the median of 
the Griewank function.
\cite{wu2019practical, jiang2020binoculars, jiang2020efficient} 
provide further comparisons of look-ahead methods
with myopic competitors, illustrating the advantage of look-ahead in general. 
Further discussion is provided in Appendix~\ref{app:Compare2OPT}.

{Numerics related to 2-step look-ahead 2-EI \eqref{eq:1qEI} are performed in Appendix~\ref{sec:MLMCReal2EI}.}

\footnotetext{According to \cite{lam2016bayesian}, the GAP = 1 results of GLASSES arise from an arbitrary choice by one optimizer to evaluate the
origin, which coincides with the minimizer of the Griewank function. Those results are excluded from the analysis.}

\section{Future Directions}\label{app:conc}

{Evaluating multi-step look-ahead acquisition functions is computationally expensive. In this work we introduced MLMC into this area and showed the benefits of using MLMC over MC. Further theoretical analysis remains to be carried out. This paper mainly focuses on the two- and three-step look-ahead settings, but it can be extended to additional steps.  
MLMC may not be able to achieve the canonical rate with 
an increasing number of look-ahead step due to the curse of dimensionality, which leads to a violation of regularity assumptions. Under mixed regularity assumptions, 
this can be solved by using multi-index Monte Carlo (MIMC) \cite{haji2016multi, jasra2023multi}, which works by considering increments of increments (of increments ...) over a grid of levels/indices instead of a simple increment at each level. 
In this case, the regularity conditions required for canonical complexity are only required on each nested MC individually. 

Randomized MLMC \cite{rhee2012new,rhee2015unbiased} and randomized MIMC can eliminate the bias and 
costly tuning ordinarily required to balance with variance,
at the cost of a potentially larger constant 
\cite{liang2023randomized}. 
It may be possible to further reduce computational complexity to MSE$^{-1/2}$ or $O(\varepsilon^{-1})$ for small $\varepsilon > 0$ by introducing multilevel quasi-Monte Carlo or multi-index quasi-Monte Carlo developed based on MLMC or MIMC and quasi-Monte Carlo \cite{caflisch1998monte, l2009monte}, which generates a sequence of low discrepancy samples for variance reduction. 
However, the improved complexity relies on the smoothness of the integrand, and performance degenerates to that of standard MC, $O(\varepsilon^{-2})$, in the non-smooth case.}

\section*{Impact Statement}

This paper presents work whose goal is to advance the field of Machine Learning. There are many potential societal consequences of our work, none which we feel must be specifically highlighted here.

\bibliography{refs}
\bibliographystyle{plain}

\newpage
\appendix
\onecolumn

\section{Acquisition Functions}\label{app:acquisition}

Some one-step look-ahead acquisition functions, such as expected improvement (EI), are analytically tractable, but the functions are generally intractable for the look-ahead steps greater than one. Monte Carlo (MC) estimation is a standard method for approximating the values. Using the re-parameterization trick, we assume the following 
\begin{align*}
    f(x, \xi;\mathcal D) &= N(\mu(x),\sigma^2(x)) = \mu(x) + \sigma(x)\xi, \quad \\
    f(x_1, \xi_1;\mathcal D_{1}(x; \xi)) &= N(\mu(x_1;x),\sigma^2(x_1;x)) = 
\mu(x_1;x) + \sigma(x_1;x)\xi_1, 
\end{align*}
where $\xi, \xi_1 \sim N(0,1)$, $\mu(\cdot;x)$ and $\sigma^2(\cdot;x)$ are the posterior mean and variance functions derived with $\mathcal D_{1}(x; \xi) = \{\mathcal D, (x, f(x,\xi;\mathcal{D}))\}$. 
{
{For simplicity, we do not carry on the reparameterization trick explicitly and keep using the notation $f(x; \mathcal{D})$ when the dependence is on $\xi$ and $f(x_1; \mathcal{D}_1(x))$ when the dependence is on $\xi$ and $\xi_1$. The superscript notations are used for realization driven by a random source. Specifically, the notation 
\begin{align*}
    f^i(x;\mathcal{D}) &= f(x,\xi^i;\mathcal{D}) \\
    f^{j}(x_1;\mathcal{D}_1(x)) &= f(x_1,\xi_1^j;\mathcal{D}_1(x;\xi))\\
    f^{ij}(x_1;\mathcal{D}_1(x)) &= f(x_1,\xi_1^j;\mathcal{D}_1(x;\xi^i))
\end{align*}
 for $i=1,2,...$ and $j=1,2,...$ is used to stand for a realization of the function. Here, $\xi^i$, $\xi_1^j$ and $\xi_1^{ij}$ are i.i.d. The data set follows the superscript notation as well.}

Even if the one-step AF is analytically tractable, we must 
approximate the two-step and three-step look-ahead AFs, which can be done with MC as follows
\begin{align}
    \alpha_{1,N}(x;\mathcal{D}) &= \mathbb E_{f(\cdot ; \mathcal D)} [ r(f,x) ] + 
 \frac{1}{N}\sum_{i=1}^{N} \max_{x_1^i}
\mathbb E_{f(\cdot ; 
\mathcal D_1^i(x))} 
 \left[ r( f,x_1^i )  \right] \, ,
\label{eq:OneStepMC}\\
\alpha_{2,N,M}(x;\mathcal{D}) &= \mathbb E_{f(\cdot ; \mathcal D)} [ r(f,x) ] + 
 \frac{1}{N}\sum_{i=1}^{N}\max_{x_1^i}\Bigg\{ 
\mathbb E_{f(\cdot ; 
\mathcal D_1^i(x))} 
 \left[ r( f,x_1^i )  \right] +  \frac{1}{M}\sum_{j=1}^{M} \max_{x_2^{ij}}
\mathbb E_{f(\cdot;\mathcal D_2^{ij}(x,x_1^i))} 
 [ r(f,x_2^{ij}) ]\Bigg\} \, ,
\label{eq:TwoStepMC}
\end{align}
where $N$ and $M$ are number of samples.

\subsection{Explicit 2- and 3-step functions}\label{sec:Exp12EI}

\textbf{Two-step look-ahead expected improvement} can be written as 
\begin{align}\label{eq:1EI}
\alpha_1(x; \mathcal{D}) :=
   \mathbb E_{f(\cdot ; \mathcal D)} \left[ r(f,x) 
+ 
 \max_{x_1} \mathbb E_{f(\cdot ; 
\mathcal D_1(x))} 
 \left[ r( f,x_1 )  \right] \right],
\end{align}
with $r(f,x) = (f(x;\mathcal{D})-f^*(\mathcal{D}))_+$, $r(f,x_1)=(f(x_1; \mathcal{D}_1(x)) - f^*(\mathcal{D}_1(x))_+$, where $f^*(\mathcal{D}_1(x)) = \max\{f^*(\mathcal{D}), f(x; \mathcal{D})\}$. The two-step look-ahead EI \eqref{eq:1EI} is analytically intractable, which can be approximated by referring to \eqref{eq:OneStepMC}.\\

\noindent
\textbf{Two-step look-ahead q-expected improvement} can be written as 
\begin{align}\label{eq:1qEI}
\alpha_1(x; \mathcal{D}) :=
   \mathbb E_{f(\cdot ; \mathcal D)} \left[ r(f,x) 
+ 
 \max_{x_1} \mathbb E_{f(\cdot ; 
\mathcal D_1(x))} 
 \left[ r( f,x_1 )  \right] \right],
\end{align}
with $r(f,x) = \max_{j=1,...,q}(f(x_j;\mathcal{D})-f^*(\mathcal{D}))_+$ and $r(f,x_1)=\max_{j=1,...,q}(f(x_{1,j}; \mathcal{D}_1(x)) - f^*(\mathcal{D}_1(x))_+$. The two-step look-ahead qEI \eqref{eq:1qEI} is analytically intractable and introduces a nested MC due to the qEI, which can be evaluated by referring to \eqref{eq:OneStepNestedMC}.

In the numerical section \ref{sec:num} of this article, we let the zero-step reward be EI and the two-step reward be qEI for convenience, i.e.
\begin{align}\label{eq:1EIqEI}
\alpha_1(x; \mathcal{D}) :=
   \mathbb E_{f(\cdot ; \mathcal D)} \left[ r_0(f,x) 
+ 
 \max_{x_1} \mathbb E_{f(\cdot ; 
\mathcal D_1(x))} 
 \left[ r_1( f,x_1 )  \right] \right],
\end{align}
with $r_0(f,x) = (f(x;\mathcal{D})-f^*(\mathcal{D}))_+$ and $r_1(f,x_1)=\max_{j=1,...,q}(f(x_{1,j}; \mathcal{D}_1(x)) - f^*(\mathcal{D}_1(x))_+$.\\

\noindent
\textbf{Three-step look-ahead expected improvement} can be written as
\begin{equation}\label{eq:2EI}
    \alpha_2(x; \mathcal{D}) :=
  \mathbb E_{f(\cdot ; \mathcal D)} \Bigg[ r(f,x) 
+  
\max_{x_1}  \mathbb E_{f(\cdot ; 
\mathcal D_1(x))} 
 \Big[ r( f,x_1)  +  
 \max_{x_2} \mathbb E_{f(\cdot;\mathcal D_2(x,x_1))} 
 [ r(f,x_2) ] 
   \Big ]
  \Bigg ]
\end{equation}
with $r(f,x) = (f(x;\mathcal{D})-f^*(\mathcal{D}))_+$, $r(f,x_1)=(f(x_{1}; \mathcal{D}_1(x)) - f^*(\mathcal{D}_1(x))_+$ and $r(f,x_2)=(f(x_2; \mathcal{D}_2(x,x_1)) - f^*(\mathcal{D}_2(x,x_1)))_+$, where $f^*(\mathcal{D}_2(x,x_1)) = \max\{f^*(\mathcal{D}_1(x)), f(x_1; \mathcal{D}_1(x))\}$. 
Referring to equation \eqref{eq:TwoStepMC}, we obtain the approximation of the Three-step look-ahead EI \eqref{eq:2EI}.

\section{Assumptions and discussion related to SAA}\label{app:saa}

The following is a set of sufficient conditions for basic SAA convergence results to hold,
which will be made throughout.

\begin{ass}\label{assumptions}{\ }
\begin{enumerate}
\item {\bf(Quadratic Growth)} For the optimum $x^*$, there exist a constant $K>0$ and an open neighborhood $V$ of $x^{*}$ such that for all $x \in \Omega\cap V$,
\begin{equation*}
  \alpha(x;\mathcal{D}) \leq    \alpha(x^*;\mathcal{D}) - K \|x^*-x\|^2.
\end{equation*}
    \item The input space $\mathcal{X}$ is compact, and  
    $\alpha(x;\mathcal{D}) \in C^1(U)$, where $U$ is a neighbourhood of $S^*$. 
    \item 
    $\mathbb{E}[\alpha(x;\mathcal{D})^2|\mathcal{D}],\  
    \mathbb{E}[\|\nabla_x \alpha(x';\mathcal{D})\|^2] < \infty$ for all $x \in \mathcal{X}$
    and $x' \in U$.
    \item Both $\alpha(\cdot;\mathcal{D})$ and $\nabla_x \alpha(\cdot;\mathcal{D})$ are Lipschitz continuous with Lipschitz constant $L(\xi)$ on the input space, which is finite in expectation.
\end{enumerate}
\end{ass}

Theorem 5.3 of \cite{shapiro2021lectures} implies that the limit point (in $N$) of any solution from $S^N$ lies in $S^*$, where $S^N$ is the set of solutions of the SAA \eqref{eq:SAA}, by assuming
\ref{assumptions}(3, 4).
If the function is convex, then a local optimizer is a global optimizer.
The rate of convergence of optimizers is based on more factors. 
We require an essential regularity condition called quadratic growth condition, defined in \ref{assumptions}(1).
Otherwise, if the function $r$ increases linearly near $x^*$, the optimizer may
converge at a faster rate \cite{kim2015guide} while the rate may be slower if it behaves like a higher-order polynomial. 

\section{Proof of Proposition \ref{prop:MCrate}}
\label{sec:proof1}

\begin{proof}
Recall $\alpha(x;\mathcal{D})$, \eqref{eq:OneStepMC} and \eqref{eq:OneStepNestedMC} as follows

\begin{align*}
\alpha(x; \mathcal{D}) &:=
   \mathbb E_{f(\cdot ; \mathcal D)} \left[ r(f,x) 
+ 
 \max_{x_1} \mathbb E_{f(\cdot ; 
\mathcal D_1(x))} 
 \left[ r( f,x_1 )  \right] \right]\\
\alpha_{N}(x;\mathcal{D}) &= \frac{1}{N}\sum_{i=1}^{N} \bigg[ r(f^i(x;\mathcal{D}))  + 
 \max_{x_1^i}
\mathbb E_{f(\cdot; 
\mathcal D_1^i(x))} 
 \left[ r( f,x_1^i )  \right] \bigg] \, \\
 \alpha_{N,M}(x;\mathcal{D}) &= \frac{1}{N}\sum_{i=1}^{N_l}
    \bigg[r(f^i(x;\mathcal{D})) + 
    \max_{x_1^i}\frac{1}{M}\sum_{j=1}^{M}
    r(f^{ij}(x_1^i;\mathcal{D}_1^i(x)))\bigg].
\end{align*}
We first prove equation \eqref{eq:nonas}. Note that 
\begin{align}
\bbE [| \alpha_{N,M}(x;\mathcal{D}) - \alpha(x;\mathcal{D})|^2] &\leq 2\bbE[|\alpha_{N}(x;\mathcal{D}) - \alpha(x;\mathcal{D})|^2] \label{eq:vb1} \\
&\quad + 2\bbE[|\alpha_{N,M}(x;\mathcal{D}) - \alpha_{N}(x;\mathcal{D})|^2],\label{eq:vb2}
\end{align}
due to $(a+b)^2 \leq 2a^2 + 2b^2$. We deal with the two terms \eqref{eq:vb1} and \eqref{eq:vb2} separately.

The first term \eqref{eq:vb1} is simply a standard MC error bounded by the following
\begin{align}
    \bbE[|\alpha_{N}(x;\mathcal{D}) - \alpha(x;\mathcal{D})|^2] &= \bbE\Bigg[\Bigg(\frac{1}{N}\sum_{i=1}^{N} \bigg[ r(f^i(x;\mathcal{D}))  + 
 \max_{x_1^i}
\mathbb E_{f(\cdot ; 
\mathcal D_1^i(x))} 
 \left[ r( f,x_1^i )  \right] \bigg]\\
 & \quad - \mathbb E_{f(\cdot ; \mathcal D)} \left[ r(f,x) 
+ 
 \max_{x_1} \mathbb E_{f(\cdot ; 
\mathcal D_1(x))} 
 \left[ r( f,x_1 )  \right] \right]\Bigg)^2\Bigg]\nonumber\\
    &= \frac{\mathsf{Var}\Big(r(f,x) 
+ 
 \max_{x_1} \mathbb E_{f(\cdot ; 
\mathcal D_1(x))} 
 \left[ r( f,x_1 )  \right]\Big)}{N}\nonumber\\
 &\leq \frac{\sigma_{\max}}{N}, \label{eq:bias}
\end{align}
where the last inequality is due to the bounded variance assumption.

For the second term \eqref{eq:vb2}, we have
\begin{align}
   \bbE [| \alpha_{N,M}(x;\mathcal{D}) - \alpha_{N}(x;\mathcal{D})|^2] & =\mathbb E\Bigg[\Bigg|\frac{1}{N}\sum_{i=1}^{N_l}
    \bigg[r(f^i(x;\mathcal{D})) + 
    \max_{x_1^i}\frac{1}{M}\sum_{j=1}^{M}
    r(f^{ij}(x_1^i;\mathcal{D}_1^i(x)))\bigg]\nonumber\\
&\quad -\frac{1}{N}\sum_{i=1}^{N} \bigg[ r(f^i(x;\mathcal{D}))  + 
 \max_{x_1^i}
\mathbb E_{f(\cdot ; 
\mathcal D_1^i(x))} 
 \left[ r( f,x_1^i )  \right] \bigg]\Bigg|^2\Bigg]\nonumber\\
&= \bbE\Bigg[ \Bigg|\frac{1}{N}\sum_{i=1}^{N_l}
    \max_{x_1^i}\frac{1}{M}\sum_{j=1}^{M}
    r(f^{ij}(x_1^i;\mathcal{D}_1^i(x))) - 
 \frac{1}{N}\sum_{i=1}^{N} \max_{x_1^i}
\mathbb E_{f(\cdot ; 
\mathcal D_1^i(x))} 
 \left[ r( f,x_1 )  \right]\Bigg|^2\Bigg]\nonumber
\end{align}

Define
\begin{equation*}
    x_{1,M}^{i,*} = \argmax_{x_1^i} 
\frac{1}{M}\sum_{j=1}^{M}
    r(f^{ij}(x_1;\mathcal{D}_1^i(x))). 
\end{equation*}
Then, we have 
\begin{align*}
    \max_{x_1^i}
\mathbb E_{f(\cdot ; 
\mathcal D_1^i(x))} 
 \left[ r( f,x_1^i )  \right] \geq  
\mathbb E_{f(\cdot ; 
\mathcal D_1^i(x))} 
 \left[ r( f,x_{1,M}^{i,*} )  \right].
\end{align*}

Thus,
\begin{align}
    &\bbE [| \alpha_{N,M}(x;\mathcal{D}) - \alpha_{N}(x;\mathcal{D})|^2] \nonumber\\
    \leq &\bbE\Bigg[ \Bigg|\frac{1}{N}\sum_{i=1}^{N}
    \frac{1}{M}\sum_{j=1}^{M}
    r(f^{ij}(x_{1,M}^{i,*};\mathcal{D}_1^i(x))) - 
 \frac{1}{N}\sum_{i=1}^{N} 
\mathbb E_{f(\cdot ; 
\mathcal D_1^i(x))} 
 \left[ r_1( f,x_{1,M}^{i,*} )  \right]\Bigg|^2\Bigg]\nonumber\\
 \leq & \frac{N}{N^2}\sum_{i=1}^{N}\bbE\Bigg[ \Bigg|
    \frac{1}{M}\sum_{j=1}^{M}
    r(f^{ij}(x_{1,M}^{i,*};\mathcal{D}_1^i(x))) - 
\mathbb E_{f(\cdot ; 
\mathcal D_1^i(x))} 
 \left[ r( f,x_{1,M}^{i,*} )  \right]\Bigg|^2\Bigg]\nonumber\\
 \leq & \frac{1}{N}\sum_{i=1}^{N}\mathsf{Var}\left(\frac{1}{M}\sum_{j=1}^{M}
    r(f^{ij}(x_{1,M}^{i,*};\mathcal{D}_1^i(x)))\right)\nonumber\\
 \leq & \frac{\sigma_{\max}}{M},\label{eq:var}
\end{align}
the first inequality is by taking the maximizer, the second inequality is by Cauchy-Schwarz inequality, and the last follows from standard SAA arguments.

Combining \eqref{eq:bias} and \eqref{eq:var}, we have $$\bbE | \alpha_{N,M}(x) - \alpha(x)|^2 \leq 
C_1(1/N
+ 1/M),$$
for some constant $C_1$.

\end{proof}

\section{MLMC estimators}

\subsection{MLMC Acquisition Function}\label{app:mlmc_af}

We define $\Delta \alpha_{N_l,M_l}(x) = \alpha_{N_l,M_l}(x) -\alpha_{N_l,M_{l-1}}(x)$ and adapt the convention that $\alpha_{N_0,M_{-1}}(x) \equiv 0$. From this, we form the estimator
\begin{equation*}
\alpha_{\sf ML}(x) := \sum_{l=0}^L\Delta \alpha_{N_l,M_l}(x)\, ,
\end{equation*}
where 
\begin{align*}
 \alpha_{N_l,M_l}(x)& = \frac{1}{N_l}\sum_{i=1}^{N_l}
    \Bigg[r(f^i(x;\mathcal{D})) + \Bigg(\max_{x_1^{ij}}\frac{1}{M_l}\sum_{j=1}^{M_l}
    r(f^{ij}(x_1^{ij};\mathcal{D}_1^i(x)))
    \Bigg)\Bigg],\\
    \alpha_{N_l,M_{l-1}}(x)& = \frac{1}{N_l}\sum_{i=1}^{N_l}
    \Bigg[r(f^i(x;\mathcal{D})) + \Bigg(\max_{x_1^{ij}}\frac{1}{M_{l-1}}\sum_{j=1}^{M_{l-1}}
    r(f^{ij}(x_1^{ij};\mathcal{D}_1^i(x))) 
    \Bigg)\Bigg].
\end{align*}
We remark that the terms $\Delta \alpha_{N_l,M_l}(x)$ for $l=0,1,...,L$ are mutually independent. We now argue that $\alpha_{\sf ML}$ is a computationally efficient proxy for the intractable limit objective $\alpha(x;\mathcal{D})$, which leads to computational benefits for approximating optimizer $x^*$. 

\subsection{Multilevel formulation of three-step look-ahead acquisition function}\label{sec:ML2Step}
The multilevel estimation for the optimizer of the three-step look-ahead AF $\alpha_2(x;\mathcal{D})$ can be formulated as 
\begin{equation*}
x^*_{\sf ML} := \sum_{l=0}^L z_{N_l,M_l} -z_{N_l,M_{l-1}}=\sum_{l=0}^L\Delta z_{N_l}\, ,
\end{equation*}
with
\begin{align}
    z_{N_l,M_l}& = \argmax_x
    \frac{1}{N_l}\sum_{i=1}^{N_l}
    \Bigg[r(f^i(x;\mathcal{D}))
    + \max_{x_1^i}\frac{1}{M_l}\sum_{j=1}^{M_l}
    \Bigg(r(f^{ij}(x_1^i;\mathcal{D}_1^i(x))) + \max_{x_2^{ij}} \mathbb E_{f(\cdot;\mathcal D_2^{ij}(x,x_1))} [ r(f, x_2^{ij})]
    \Bigg)\Bigg],\label{eq:MLfine2}\\
    z_{N_l,M_{l-1}}& = \argmax_x\frac{1}{N_l}\sum_{i=1}^{N_l}
    \Bigg[r(f^i(x;\mathcal{D}))
    + \max_{x_1^i}\frac{1}{M_{l-1}}\sum_{j=1}^{M_{l-1}}
    \Bigg(r(f^{ij}(x_1^i;\mathcal{D}_1^i(x))) + \max_{x_2^{ij}} \mathbb E_{f(\cdot;\mathcal D_2^{ij}(x,x_1))} [ r(f,x_2^{ij})]
    \Bigg)\Bigg].\label{eq:MLcoarse2}
\end{align}
Here we use the convention $z_{N_0,M_{-1}}\equiv 0$. To leverage the MLMC, we also require the subsampling of $M_{l-1}$ samples from $M_l$ similar to the two-step look-ahead one. The antithetic approach of the three-step function involves replacing the coarse estimator \eqref{eq:MLcoarse2} with 

\begin{align*}
    z_{N_l,M_{l-1}}^A & = \argmax_x\frac{1}{N_l}\sum_{i=1}^{N_l}
    \Bigg[r(f^i(x;\mathcal{D}))\\ 
    &\quad + \frac{1}{2}\Bigg( \max_{x_1^i}\frac{1}{M_{l-1}}\sum_{j=1}^{M_{l-1}}
    \bigg(r(f^{ij}(x_1^i;\mathcal{D}_1^i(x))) + \max_{x_2^{ij}} \mathbb E_{f(\cdot;\mathcal D_2^{ij}(x,x_1))} [ r(f,x_2^{ij})]\bigg) \\
    &\quad +\max_{x_1^i}\frac{1}{M_{l-1}}\sum_{j=M_{l-1}+1}^{M_{l}}
    \bigg(r(f^{ij}(x_1^i;\mathcal{D}_1^i(x))) + \max_{x_2^{ij}} \mathbb E_{f(\cdot;\mathcal D_2^{ij}(x,x_1))} [ r(f, x_2^{ij})]\bigg)
    \Bigg)\Bigg],
\end{align*}
where $z^A_{N_0,M_{-1}}\equiv 0$. The sampling step is the same as the two-step look-ahead one.

\section{Value function and maximizer corollaries}\label{app:value_max}

Theorem \ref{thm:qFunc} provides a convergence rate for q-functions. We can then use this result to establish the convergence of value functions:

\begin{cor}[Value Function Convergence] \label{cor:value}
    For $v^\star(\mathcal{D}) := \max_x \alpha(x;\mathcal D)$ and $v_{\sf ML}:=\max_x \alpha_{\sf ML}(x)$
        \begin{equation*}
        \mathbb E \big[ | v_{\sf ML} - v^\star(\mathcal{D}) |^2  \big] \leq C(L+1)\left(\frac{V_0}{N_0} + \sum_{l=1}^L\frac{1}{N_lM_l}\right) + \frac{C}{M_L}
    \end{equation*}
    and thus under the parameter choices \eqref{eq:params}, a root-mean-square error of $\varepsilon^2$ can be achieved with a sample complexity of $O(\varepsilon^{-2}|\log \varepsilon|^3)$.
\end{cor}
The proof of Corollary \ref{cor:value} is given in Appendix 
\ref{sec:corvalue}. 
Assuming quadratic growth at the optimizer, we can prove convergence to the optimal decision point. 

\begin{cor}[Maximizer Convergence]\label{cor:argmax}
    Assume \ref{assumptions} and \ref{ass:IncVar}.
    For $\tilde{x}^*_{\sf ML} \in \argmax_x \alpha_{\sf ML}(x)$ and $x^* = \argmax_x \alpha(x;\mathcal D) $
    \begin{equation*}
        \mathbb E [ \| \tilde{x}^*_{\sf ML} - x^*\|^2] 
        \leq 
        {{C(L+1)\left(\frac{V_0}{N_0} + \sum_{l=1}^L\frac{1}{N_lM_l}\right)+ \frac{1}{M_L} } }
    \end{equation*}
    and thus under the parameter choices \eqref{eq:params}, a mean-square error of {$\varepsilon^2$ can be achieved with a sample complexity of $O(\varepsilon^{-2}|\log \varepsilon|^3)$}.
\end{cor}

The proof of Corollary \ref{cor:argmax} is given in Appendix \ref{sec:corargmax}. 

{It is difficult to numerically verify Theorem \ref{thm:qFunc} due to the inner sup. Numerical verification of complexity for Corollary \ref{cor:value} and \ref{cor:argmax} are in Appendix \ref{sec:NumCoro}.}

\section{Assumptions and discussion related to MLMC}\label{app:mlmc}

We let 
\begin{align}
    \Delta_M(x) &=  \max_{x_1}\frac{1}{M}\sum_{j=1}^{M}
    r(f^{j}(x_1;\mathcal{D}_1(x))) - \max_{x_1} \mathbb E_{f(\cdot ; 
\mathcal D_1(x))} 
\left[ r( f,x_1)  \right] \, , \\
{ \nabla \Delta_M(x)} &=  
{\nabla_x \max_{x_1}\frac{1}{M}\sum_{j=1}^{M}
    r(f^{j}(x_1;\mathcal{D}_1(x)))}{- \nabla_x \max_{x_1} \mathbb E_{f(\cdot ; 
\mathcal D_1(x))} \left[ r( f,x_1)  \right]} \, .
\end{align}

{Notice that $\Delta_M(x)$ and $\nabla\Delta_M(x)$ summarize the error in the sample average approximation of the inner expectation and its derivative, respectively}. Given the existence of comparable results on SAA, we make the following assumption:

\begin{ass}\label{ass:IncVar}
For some constant $C>0$ and $\forall x \in \mathcal{X}$,
\begin{align*}
 \sup_x \big| \mathbb E \big[ \Delta_M(x)\big] \big| \leq \frac{C}{M^{\frac{1}{2}}}
 \qquad \text{and} \qquad 
\mathbb E \left[ \sup_x \Big|\Delta_M(x) - \mathbb E [\Delta_M (x)] \Big|^2  \right] \leq \frac{C}{M},
\end{align*}
\begin{align*}
 {\sup_x \big| \mathbb E \big[ \nabla\Delta_M(x)\big] \big| \leq \frac{C}{M^{\frac{1}{2}}}
 \qquad \text{and} \qquad 
\mathbb E \left[ \sup_x \Big|\nabla\Delta_M(x) - \mathbb E [\nabla\Delta_M (x)] \Big|^2  \right] \leq \frac{C}{M}}.
\end{align*}
\end{ass}
\begin{rem}\label{rem:inc}
The above Assumption is found to hold for Sample Average Approximation for any given value of $x$, See Page 165 of \cite{shapiro2021lectures}. Here we ask for standard SAA approximation results to hold over all values of $x$ taken in our outer sample. 
Also,
    if $\max_{x_1}\frac{1}{M}\sum_{j=1}^{M}
    r(f^{j}(x_1;\mathcal{D}_1(x)))$ is square integrable, Assumption \ref{ass:IncVar} is satisfied by CLT. The square integrability of $\max_{x_1}\frac{1}{M}\sum_{j=1}^{M}
    r(f^{j}(x_1;\mathcal{D}_1(x)))$ can be obtained by the square integrability of underlying Gaussian random variables, following the similar argument of showing uniform integrability as Theorem 6.1.6 of \cite{borovkov1999probability}. The underlying GP is square integrable by choosing a proper integrable kernel function.
\end{rem}

We state an intermediate result in Lemma \ref{lem:cross}, which shows the variance of the increments satisfies the MLMC assumption and leads to Theorem \ref{thm:qFunc}.
\begin{lem}\label{lem:cross}
    Assume \ref{ass:IncVar} and $M_l=cM_{l-1}$ for some positive integer $c$. Then, for some constant $C > 0$,
    \begin{equation*}
        \mathbb E[\sup_x |\Delta \alpha_{N_l,M_l}(x) - \mathbb E[\Delta \alpha_{N_l,M_l}(x)]|^2] \leq \frac{C}{N_lM_l}.
    \end{equation*}
\end{lem}

The proof of Lemma \ref{lem:cross} is in Appendix \ref{sec:ProofLemma1}.

\subsection{Additional assumption for Theorem \ref{thm:main}}

The additional assumption required for the implementable algorithm is now given.

{
\begin{ass}\label{ass:qImpOpt}
    Assume for $\beta > 0$ and $x\in\mathcal{X}$, the bound
    \begin{align*}
        \mathbb E[|\Delta \alpha_{N_l,M_l}(x) - \Delta \alpha_{\infty,M_l}(x)|^2]
        \leq \frac{C}{N_lM_l^{\beta}},
    \end{align*}
    implies
    \begin{equation*}
        \mathbb E[\|\Delta z_{N_l, M_l} - \Delta z_{\infty,M_l}\|^2] \leq \frac{C}{N_lM_l^{\beta}}.
    \end{equation*}
\end{ass}}
{
\begin{rem}
    The implication Assumption \ref{ass:qImpOpt} holds for differences,
    e.g. 
    $\mathbb E[\|\Delta z_{N_l, M_l}\|] \leq 
    \sup_{x\in\mathcal{X}}\mathbb E[| \nabla\Delta \alpha_{N_l,M_l}(x)|]$,
    under suitable assumptions, following the argument in the proof of
    Corollary \ref{cor:argmax}.
    We have validated it numerically for differences of differences in Appendix \ref{sec:NumMLMCRates},
    but it remains to be proven. 
\end{rem}}

\section{Proof of lemmas and theorems}
\label{sec:proof2}

Recall
\begin{align*}
    \alpha(x; \mathcal{D}) &=
   \mathbb E_{f(\cdot ; \mathcal D)} \left[ r(f,x) 
+ 
 \max_{x_1} \mathbb E_{f(\cdot ; 
\mathcal D_1(x))} 
 \left[ r( f,x_1 )  \right] \right],\\
 \alpha_{N,M}(x; \mathcal{D})& = \frac{1}{N}\sum_{i=1}^{N}
    \Bigg[r(f^i(x;\mathcal{D}))
    + \Bigg(\max_{x_1^i}\frac{1}{M}\sum_{j=1}^{M}
    r(f^{ij}(x_1^i;\mathcal{D}_1^i(x)))
    \Bigg)\Bigg].
\end{align*}
Define
\begin{align*}
  \alpha_{N,\infty}(x; \mathcal{D})& = \frac{1}{N}\sum_{i=1}^{N}
    \Bigg(r(f^i(x;\mathcal{D}))
    + \max_{x_1^i} \mathbb E_{f(\cdot ; 
\mathcal D^i_1(x))} 
 \left[ r( f,x_1^i )  \right]
    \Bigg),\\
    \alpha_{\infty,M}(x; \mathcal{D}) &=
   \mathbb E_{f(\cdot ; \mathcal D)} \left[ r(f,x) 
+ 
 \max_{x_1}\frac{1}{M_{l}}\sum_{j=1}^{M_{l}}
    r(f^{j}(x_1;\mathcal{D}_1(x)))\right].
\end{align*}

We first state a lemma.
\begin{lem}\label{lem:IoIRate}
    Assume \ref{ass:IncVar}. For some constant $C>0$,
    \begin{equation*}
        \mathbb E\big[\sup_{x}\big |\alpha_{N,M}(x) - \alpha_{N,\infty}(x) - \mathbb E[\alpha_{N,M}(x) - \alpha_{N,\infty}(x)]\big |^2\big] \leq \frac{C}{NM}.
    \end{equation*}
\end{lem}

\begin{proof}
    We have
    \begin{align*}
        &\mathbb E\big[\sup_{x}\big |\alpha_{N,M}(x) - \alpha_{N,\infty}(x) - \mathbb E[\alpha_{N,M}(x) - \alpha_{N,\infty}(x)]\big |^2\big]\\
        =&\mathbb E\Bigg[\sup_{x}\Bigg|\frac{1}{N}\sum_{i=1}^{N}\Bigg(\max_{x_1^i}\frac{1}{M}\sum_{j=1}^{M}
    r(f^{ij}(x_1^i;\mathcal{D}_1^i(x))) - \max_{x_1^i} \mathbb E_{f(\cdot ; 
\mathcal D^i_1(x))} 
 \left[ r( f,x_1^i )  \right]\\
 & - \mathbb E\bigg[\max_{x_1^i}\frac{1}{M}\sum_{j=1}^{M}
    r(f^{ij}(x_1^i;\mathcal{D}_1^i(x))) - \max_{x_1^i} \mathbb E_{f(\cdot ; 
\mathcal D^i_1(x))} 
 \left[ r( f,x_1^i )  \right]\bigg]\Bigg)\Bigg|^2\Bigg]\\
 =&\frac{1}{N}\mathbb E\Bigg[\sup_{x}\Bigg(\max_{x_1^i}\frac{1}{M}\sum_{j=1}^{M}
    r(f^{ij}(x_1^i;\mathcal{D}_1^i(x))) - \max_{x_1^i} \mathbb E_{f(\cdot ; 
\mathcal D^i_1(x))} 
 \left[ r( f,x_1^i )  \right]\\
 & - \mathbb E\bigg[\max_{x_1^i}\frac{1}{M}\sum_{j=1}^{M}
    r(f^{ij}(x_1^i;\mathcal{D}_1^i(x))) - \max_{x_1^i} \mathbb E_{f(\cdot ; 
\mathcal D^i_1(x))} 
 \left[ r( f,x_1^i )  \right]\bigg]\Bigg)^2\Bigg]\\
 =& \frac{1}{N} \mathbb E \left[ \sup_x \Big|\Delta_M(x) - \mathbb E [\Delta_M (x)] \Big|^2  \right] 
 \\
 \leq & \frac{C}{NM}.
    \end{align*}
In the final inequality, we apply Assumption \ref{ass:IncVar}.
\end{proof}

\subsection{Proof of Lemma \ref{lem:cross}}\label{sec:ProofLemma1}
Now, we prove Lemma \ref{lem:cross}, which is a direct result of Lemma \ref{lem:IoIRate}.

\begin{proof}
    Note that 
    \begin{align*}
        \Delta \alpha_{N_l,M_l}(x) &= \alpha_{N_l,M_l}(x) - \alpha_{N_l,M_{l-1}}(x)\\
        &= \alpha_{N_l,M_l}(x) - \alpha_{N_l,\infty}(x) - \alpha_{N_l,M_{l-1}}(x) + \alpha_{N_l,\infty}(x).
    \end{align*}
    Thus, 
    \begin{align*}
        \mathbb E[\sup_x|\Delta \alpha_{N_l,M_l}(x) - \mathbb E[\Delta \alpha_{N_l,M_l}(x)]|^2] &= \mathbb E\big[\sup_x\big|\alpha_{N_l,M_l}(x) - \alpha_{N_l,\infty}(x) - \mathbb E[\alpha_{N_l,M_l}(x) - \alpha_{N_l,\infty}(x)]\\ 
        &\quad - \alpha_{N_l,M_{l-1}}(x) + \alpha_{N_l,\infty}(x) + \mathbb E[\alpha_{N_l,M_{l-1}}(x) - \alpha_{N_l,\infty}(x)]\big|^2\big]\\
        &\leq 2\mathbb E\big[\sup_x\big|\alpha_{N_l,M_l}(x) - \alpha_{N_l,\infty}(x) - \mathbb E[\alpha_{N_l,M_l}(x) - \alpha_{N_l,\infty}(x)]\big|^2\big]\\ 
        &\quad + 2\mathbb E\big[\sup_x\big|\alpha_{N_l,M_{l-1}}(x) - \alpha_{N_l,\infty}(x) - \mathbb E[\alpha_{N_l,M_{l-1}}(x) - \alpha_{N_l,\infty}(x)]\big|^2\big]\\
        &\leq \frac{2C}{N_lM_l} + \frac{2C}{N_lM_{l-1}}\\
        &\leq \frac{C}{N_lM_l},
    \end{align*}
    where the second inequality is obtained by applying Lemma \ref{lem:IoIRate} and the last inequality is by $M_{l} = cM_{l-1}$ for some positive integer $c$ and relabeling the constant $C$.
\end{proof}

\subsection{Proof of Theorem \ref{thm:qFunc}}\label{sec:ProofqFunc}
The proof of Theorem \ref{thm:qFunc} is given below.
{The proof is only given for $\alpha$, however,
given the assumptions the proof for $\nabla \alpha$ 
follows in exactly the same fashion.}
\begin{proof}
To begin, the error can be divided into its bias and variance components as follows:
    \begin{align}
        \mathbb{E}[\sup_x \|\alpha_{\sf ML}(x) - \alpha(x;\mathcal{D})\|^2] &\leq 2\mathbb{E}[\sup_x \|\alpha_{\sf ML}(x) - \mathbb{E}[\alpha_{\sf ML}(x)]\|^2]\label{eq:firstq}\\
        &\quad +  2 \sup_x |\mathbb E[\alpha_{\sf ML}(x)] - \alpha(x;\mathcal{D})|^2\label{eq:secondq}
    \end{align}
    We will proceed to bound \eqref{eq:firstq} and \eqref{eq:secondq}. 
    We bound the term \eqref{eq:firstq}. 
 Notice that, by Cauchy-Schwartz
 \[
|\alpha_{\sf ML}(x) - \mathbb E [\alpha_{\sf ML}(x)]  |^2
=
\Big| \sum_{l=0}^L \Delta \alpha_{N_l, M_l}(x) - \mathbb E [\Delta \alpha_{N_l, M_l}(x)]  \Big|^2
 \leq (L+1)  \sum_{l=0}^L \Big| \Delta \alpha_{N_l, M_l}(x) - \mathbb E [\Delta \alpha_{N_l, M_l}(x)]  \Big|^2 \, .
  \]
 Thus, we have
    \begin{align}
        \mathbb{E}[\sup_x \|\alpha_{\sf ML}(x) - \mathbb{E}[\alpha_{\sf ML}(x)]\|^2] 
        &\leq (L+1) \sum_{l=0}^L \mathbb{E}[\sup_x\|\Delta \alpha_{N_l, M_l}(x) - \mathbb{E}[\Delta \alpha_{N_l, M_l}(x)]\|^2]\nonumber\\
        &\leq C (L+1)\left(\frac{V_0}{N_0} + \sum_{l=1}^L\frac{1}{N_lM_l}\right).\label{eq:firstBound}
    \end{align}
    The first line above is due to independence. For $l=0$, we have the variance of $\alpha_{N_0,M_0}$ as $\mathbb E[\|\alpha_{N_0,M_0}-\mathbb E[\alpha_{N_0,M_0}]\|^2]=V_0/N_0$. The equation \eqref{eq:firstBound} follows from Lemma \ref{lem:cross}.

    Finally, we deal with the term \eqref{eq:secondq}. First we note that $\alpha_{\sf ML}(x)$ is an unbiased estimator of $\alpha_{\infty,L}(x)$ since
    \begin{align*}
        \mathbb E[\alpha_{\sf ML}(x)] &= \mathbb E\Big[\sum_{l=0}^L\Delta \alpha_{N_l, M_l}(x)\Big]\\
        &= \mathbb E[\alpha_{N_0,M_0}(x)] + \sum_{l=1}^L(\mathbb E[\alpha_{N_l,M_l}(x)] - \mathbb E[\alpha_{N_l,M_{l-1}}(x)])\\
        &= \mathbb E[\alpha_{N_0,M_0}(x)] + \sum_{l=1}^L(\mathbb E[\alpha_{N_l,M_l}(x)] - \mathbb E[\alpha_{N_{l-1},M_{l-1}}(x)])\\
        &= \mathbb E[\alpha_{N_L,M_L}(x)]\\
        &= \alpha_{\infty,M_L}(x).
    \end{align*}
    For the third line, above, note that $\mathbb E[\alpha_{N_l,M_{l-1}}(x)] = \mathbb E[\alpha_{N_{l-1},M_{l-1}}(x)]$. In the fourth line, we cancel terms in the interpolating sum.
    Thus
    \begin{equation*}
        |\mathbb E[\alpha_{\sf ML}(x)] - \alpha(x;\mathcal{D})| = | \alpha_{\infty,M_L}(x) - \alpha(x;\mathcal{D})|.
    \end{equation*}
   And so, 
    \begin{align}
      &  \sup_x |\alpha_{\infty,M_L}(x) - \alpha(x;\mathcal{D})|^2 \notag \\
      &= \sup_x\Bigg|
   \mathbb E_{f(\cdot ; \mathcal D)} \bigg[ r(f,x) 
+ 
 \max_{x_1}\frac{1}{M_{L}}\sum_{j=1}^{M_{L}}
    r(f^{j}(x_1;\mathcal{D}_1(x)))\bigg]  - 
   \mathbb E_{f(\cdot ; \mathcal D)} \bigg[ r(f,x) 
+ 
 \max_{x_1} \mathbb E_{f(\cdot ; 
\mathcal D_1(x))} 
 \left[ r( f,x_1 )  \right] \bigg]\Bigg|^2\nonumber\\
        &= \sup_x \Bigg|
   \mathbb E_{f(\cdot ; \mathcal D)} \bigg[
 \max_{x_1}\frac{1}{M_{L}}\sum_{j=1}^{M_{L}}
    r(f^{j}(x_1;\mathcal{D}_1(x))) - 
 \max_{x_1} \mathbb E_{f(\cdot ; 
\mathcal D_1(x))} 
 \left[ r( f,x_1 )  \right] \bigg]\Bigg|^2 \nonumber\\
 & = \sup_x |\mathbb E \big[ \Delta_{M_L}(x) \big] |^2 \notag
 \\
 &\leq \frac{C}{M_L} \label{eq:secondBound},
    \end{align}
The final inequality follows by Assumption \ref{ass:IncVar}.
Applying \eqref{eq:firstBound} and \eqref{eq:secondBound} to \eqref{eq:firstq} and \eqref{eq:secondq} gives
\begin{equation*}
    \mathbb E \big[ \sup_x \|\alpha_{\sf ML}(x) - \alpha(x;\mathcal{D})\|^2 \big] \leq C(L+1)\left(\frac{V_0}{N_0} +  \sum_{l=1}^{L}\frac{1}{N_lM_l}\right) + \frac{C}{M_L}.
\end{equation*}
Thus the required the bound in \ref{thm:qFunc} holds. 

Setting $L = 2|\log\varepsilon|/\log 2$ gives the term $C/M_L$ of $O(\varepsilon^2)$. Given the cost assumption, we have total computational cost Cost$=\sum_{l=0}^{L}M_lN_l$. We can minimise this cost for a fixed variance $\frac{V_0}{N_0} + \sum_{l=1}^L\frac{1}{N_lM_l} = \varepsilon^2/(L+1)$ using Lagrange multipliers. See Lemma \ref{lem:costmin} for a proof. Specifically, Lemma \ref{lem:costmin}
(with $c_l=M_l$, $l=0,...,L$, $v_0=V_0$ and $v_l=1/M_l$, $l=1,...,L$) gives 
\begin{equation*}
    N_0 = (L+1)\varepsilon^{-2}K_L\left(\frac{V_0}{M_0}\right)^{1/2}\quad \text{and} \quad N_l=(L+1)\varepsilon^{-2}K_L\frac{1}{M_l}, \quad \text{for}\quad l \geq 1,
\end{equation*}
for
\begin{equation*}
    K_L = (V_0M_0)^{1/2} + L,
\end{equation*}
and hence COST $=(L+1)\varepsilon^{-2}K_L^2 = O(|\log \varepsilon|^3 / \varepsilon^2 )$ since $K_L={O}(|\log\varepsilon|)$ and $L = O(|\log\varepsilon|)$.
Finally, observe that if we round up terms, I.e.
$
    N_0 = \Big\lceil \varepsilon^{-2}K_L\left(\frac{V_0}{M_0}\right)^{1/2} \Big\rceil $ \text{and} $ N_l= \Big\lceil \varepsilon^{-2}K_L\frac{1}{M_l} \Big\rceil, \quad \text{for}\quad l \geq 1,
$
then the additional cost added is $\sum_{l=0}^L M_l \leq 2^{L+1} \leq \frac{2}{\epsilon^2}$. Thus the total cost of the algorithm is of order $O (  \varepsilon^{-2}|\log \varepsilon|^3)$.
\end{proof}

\begin{lem}\label{lem:costmin}
The cost minimization 
\begin{align*}
  \text{minimize} \quad \sum_{l=0}^L n_l c_l \quad \text{subject to} \quad \sum_{l=0}^L \frac{v_l}{n_l} \leq \epsilon^2
\quad \text{over} \quad n_l \geq 0, \;\; l=0,...,L,
\end{align*}
has optimal solution and the optimal cost
\begin{align*}
  n^\star_l = \frac{1}{\epsilon^2} \sqrt{\frac{v_l}{c_l}} \left( \sum_{l'=0}^L \sqrt{v_{l'}c_{l'}} \right)\qquad \text{and} \qquad   C^\star = \frac{1}{\epsilon^2} \left( \sum_{l'=0}^L \sqrt{v_{l'}c_{l'}}
\right)^2 .
\end{align*}
\end{lem}
\begin{proof}
    The Lagrangian of the above optimization problem is:
    \begin{equation*}
        \mathcal{L}(n;\lambda) = \sum_{l=0}^L n_l c_l  + \lambda \left(
\sum_{l=0}^L \frac{v_l}{n_l} - \epsilon^2
        \right) \, .
    \end{equation*}
    Thus 
    \[
        c_l - \frac{\lambda v_l}{n_l^2} = 0 ,\qquad  \text{and thus} \qquad n_l = \sqrt{\frac{\lambda v_l}{c_l}} \, .
    \]
For primal feasibility, we require
\[
\epsilon^2 =  \frac{v_l}{n_l} =  \sum_{l=0}^L \sqrt{\frac{c_l v_l}{\lambda }}, 
\qquad \text{thus} \qquad 
\lambda = \left( 
    \frac{1}{\epsilon^2} \sum_{l=0}^L \sqrt{c_lv_l} 
\right)^2\, .
\]
Thus we have, as required,
\[
n_l =  \left( \frac{1}{\epsilon^2} \sum_{l=0}^L \sqrt{c_lv_l} \right) \times \sqrt{ \frac{v_l}{c_l}} \, ,
\qquad
\text{and}
\qquad
\sum_{l=0}^L n_l c_l =  \frac{1}{\epsilon^2} \left(\sum_{l=0}^L \sqrt{c_lv_l} \right)^2 \, .
\]
\end{proof}

\subsection{Proof of Corollary \ref{cor:value}}\label{sec:corvalue}

The proof of Corollary \ref{cor:value} is a consequence of the following standard lemma.

\begin{lem}\label{lem:maxy}
    For two functions $\alpha_1(x)$ and $\alpha_2(x)$
    \[
| \max_x \alpha_1(x) - \max_x \alpha_2(x) | \leq \sup_x | \alpha_1(x) - \alpha_2(x) |
    \]
\end{lem}

\begin{proof}
    Let $x_1^\star \in \argmax \alpha_1(x)$ then
    \[
 \max_x \alpha_1(x) - \max_x \alpha_2(x) 
 = \alpha_1(x_1^\star ) - \max_x \alpha_2(x) 
 \leq 
 \alpha_1(x_1^\star ) - \alpha_2(x_1^\star )
 \leq 
 \sup_x | \alpha_1(x) - \alpha_2(x) |
    \]
By a symmetric argument we have $\max_x \alpha_2(x) - \max_x \alpha_1(x)\leq \sup_x | \alpha_1(x) - \alpha_2(x) | $ and thus the result holds.
\end{proof}

\begin{proof}[Proof of Corollary \ref{cor:value}]
By Lemma \ref{lem:maxy} and Theorem \ref{thm:qFunc}
\[
            \mathbb E \big[ | v_{\sf ML} - v^\star(\mathcal{D}) |^2  \big] 
            \leq 
            \mathbb E\Big[\sup_x |\alpha_{\sf ML}(x) - \alpha(x;\mathcal{D})|^2\Big]
            \leq C(L+1)\left(\frac{V_0}{N_0} + \sum_{l=1}^L\frac{1}{N_lM_l}\right) + \frac{C}{M_L}
\]
\end{proof}

\subsection{Proof of Corollary \ref{cor:argmax}.}\label{sec:corargmax}

\begin{proof}[Proof of Corollary \ref{cor:argmax}.]
    By the quadratic growth condition on $\alpha(x;\mathcal{D})$ in Assumption \ref{assumptions}, we have for some $C>0$,  
   \begin{align*}
        C\|\tilde{x}_{\sf ML}^* - x^*\|^2 &\leq   \alpha(x^*;\mathcal{D})- \alpha(\tilde{x}_{\sf ML}^*;\mathcal{D}) \\
        &= \alpha(x^*;\mathcal{D})- \alpha(\tilde{x}_{\sf ML}^*;\mathcal{D}) + \alpha_{\sf ML}(\tilde{x}_{\sf ML}^*) - \alpha_{\sf ML}(\tilde{x}_{\sf ML}^*)\\
        &\leq\alpha(x^*;\mathcal{D})- \alpha(\tilde{x}_{\sf ML}^*;\mathcal{D})  + \alpha_{\sf ML}(\tilde{x}_{\sf ML}^*) - \alpha_{\sf ML}(x^*)\\
        &= \delta_{\sf ML}(\tilde{x}_{\sf ML}^*;\mathcal{D}) - \delta_{\sf ML}(x^*;\mathcal{D})\\
        &\leq \nabla\delta_{\sf ML}(y;\mathcal{D}) \cdot (\tilde{x}_{\sf ML}^*-x^*)\\
        &\leq |\nabla\delta_{\sf ML}(y;\mathcal{D})|\|\tilde{x}_{\sf ML}^*-x^*\|\\
        &={\|\nabla\alpha_{\sf ML}(y) - \nabla\alpha(y;\mathcal{D})\|}\|\tilde{x}_{\sf ML}^*-x^*\|,
    \end{align*}
    
{Dividing through by $\| x^*_{\sf ML} - x^*\|$, squaring and taking sup and expectation, 
we have  
\begin{align*}
 C\mathbb E [ \| x^*_{\sf ML} - x^*\|^2]
& \leq \mathbb E [ \sup_y\|\nabla\alpha_{\sf ML}(y) - \nabla\alpha(y;\mathcal{D})\|^2] 
   \\
&
\leq 
{C(L+1)\left(\frac{V_0}{N_0} + \sum_{l=1}^L\frac{1}{N_lM_l}\right) + \frac{C}{M_L}}  \, .
\end{align*}
Theorem \ref{thm:qFunc} is applied directly. Applying the parameters, \eqref{eq:params} we see that a mean-squared-error of $\varepsilon^2$ is reached with a sample complexity of 
$O(\varepsilon^{-2}|\log \varepsilon|^3) $.} 
\end{proof}

\subsection{Proof of Theorem \ref{thm:main}}\label{sec:ProofOptimizer}
{We now provide the proof of the main MLMC Theorem \ref{thm:main}, reproduced here for convenience.}
{\begin{theorem}\label{thm:mainapp}
Assume \ref{assumptions}, \ref{ass:IncVar} and \ref{ass:qImpOpt}. Let $M_l = 2^l$.
    The estimator $x_{\sf ML}^*$ is such that, for $x\in\mathcal{X}$,
    \begin{equation*}
        \bbE [\| x^*_{\sf ML} - x^*\|^2]\leq C\left(\frac{V_0}{N_0} + \sum_{l=1}^L\frac{1}{N_lM_l^{\beta}}\right) + \frac{C}{M_L}
    \end{equation*}
    for $\beta \geq 1$ and some constant C. Consequently taking $L = O(|\log\varepsilon|)$,
\begin{equation*}
    N_0 = \varepsilon^{-2}K_L\left(\frac{V_0}{M_0}\right)^{1/2}\ \text{and} \ \  N_l=\varepsilon^{-2}K_L\frac{1}{M_l},
\end{equation*}
for $l \geq 1 $ with $K_L = (V_0M_0)^{1/2} + L,$
the MLMC estimator \eqref{eq:mlmcest} gives the following error estimate, for some $x^* \in S^*$,
\begin{equation*}
    \bbE [\| x^*_{\sf ML} - x^*\|^2] \leq C \varepsilon^2 \, ,
\end{equation*}
for a complexity of $O(\varepsilon^{-2}(\log\varepsilon)^2)$.
\end{theorem}}

\begin{proof}

We have
\begin{align*}
    \mathbb{E}[\|x^*_{\sf ML} - x^*\|^2] &\leq \mathbb{E}\bigg[\bigg\|\sum_{l=0}^L \Delta z_{N_l,M_l} - \sum_{l=0}^L (z_{\infty, M_l}-z_{\infty, M_{l-1}}) + \sum_{l=0}^L (z_{\infty, M_l}-z_{\infty, M_{l-1}})- x^*\bigg\|^2\bigg]\\
    &\leq C \sum_{l=0}^L\mathbb{E}\bigg[\bigg\|\Delta z_{N_l,M_l} - (z_{\infty, M_l} - z_{\infty, M_{l-1}})\bigg\|^2\bigg] + C\mathbb E[\|z_{\infty, M_L} - x^*\|^2]\\
    &\leq C\left(\frac{V_0}{N_0} + \sum_{l=1}^L\frac{1}{N_lM_l^{\beta}}\right) + \frac{C}{M_L}
\end{align*}
for $\beta=1$, where the last line is due to Assumption \ref{ass:qImpOpt}, Lemma \ref{lem:cross} and SAA results. The bias of ${O}(1/M_L) = {O}(2^{-L})$ requires one to choose $L = O(|\log\varepsilon|)$ for the desired MSE. We have total computational cost Cost$=\sum_{l=0}^{L}M_lN_l$. Minimising the cost for a fixed variance $\frac{V_0}{N_0} + \sum_{l=1}^L\frac{1}{N_lM_l} = \varepsilon^2$ using Lagrange multipliers as Lemma \ref{lem:costmin} gives 
\begin{equation*}
    N_0 = \varepsilon^{-2}K_L(V_0M_0)^{1/2}\quad \text{and} \quad N_l=\varepsilon^{-2}K_LM_l^{-(1+\gamma)/2}, \quad \text{for}\quad l \geq 1,
\end{equation*}
for
\begin{equation*}
    K_L = (V_0M_0)^{1/2} + C\sum_{l=1}^LM_l^{-(1-1)/2} = (V_0M_0)^{1/2} + CL,
\end{equation*}
and hence COST $=\varepsilon^{-2}K_L^2$ with $K_L={O}(|\log\varepsilon|)$. We recover the borderline complexity.
\end{proof}

\section{Numerical Results}

\subsection{Inner and outer Monte Carlo rates}\label{sec:NumMCRates}

We test the inner and outer rates using the 1D toy example \ref{eq:toy} introduced before. Here, we apply \eqref{eq:OneStepMC} with EIs and assume that we can only approximate the inner EI using MC estimation, i.e., we have
\begin{align}\label{eq:InnerEstimator}
    \alpha_{1,N,M}(x) = {\sf EI}(x|\mathcal{D}) + 
\frac{1}{N}\sum_{i=1}^{N}\max_{x_1}\frac{1}{M}\sum_{j=1}^M(f^j(x_1) - f^*(\{ \mathcal{D}_n, (x,f^i(x))\}))_+. 
\end{align}
The reason for using \eqref{eq:OneStepMC} is that we have a benchmark for the inner expectation.\\

\noindent
\textbf{Rate with respect to $N$:} We test the convergence with respect to $N$ using equation \eqref{eq:InnerEstimator} with inner MC replaced by the analytical solution of EI. The reference solution is computed with $N = 2^{12}$. According to Figure \ref{fig:OneEIOuter}, the maximizer converges with $O(N^{-1})$.\\ 

\noindent
\textbf{Rate with respect to $M$:} For the convergence with respect to $M$, we fix the number of sample $N=2^5$ for the outer Monte Carlo of equation \eqref{eq:InnerEstimator} and then let $M$ varies. Note that the approximation is sensitive to the base sample, so we fixed the base samples ($\xi^i$) for the outer MC. Figure \ref{fig:OneEIInner} illustrates that the maximizer converges with $O(M^{-1})$.

\begin{figure}
    \centering
    \begin{subfigure}{0.45\textwidth}
    \includegraphics[width=\linewidth]{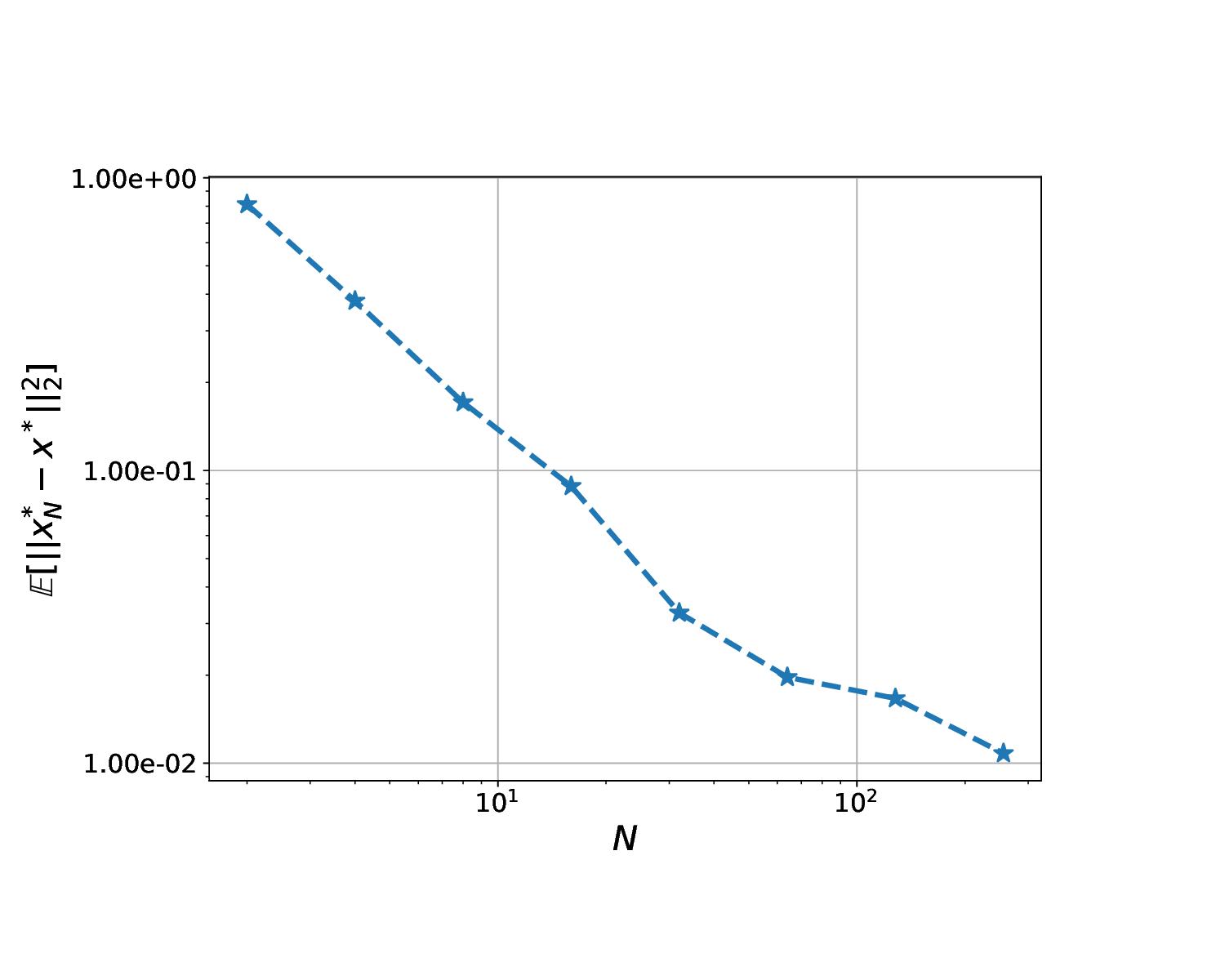}
    \caption{}
    \label{fig:OneEIOuter}
    \end{subfigure}
    \hfill
    \begin{subfigure}{0.45\textwidth}
    \includegraphics[width=\linewidth]{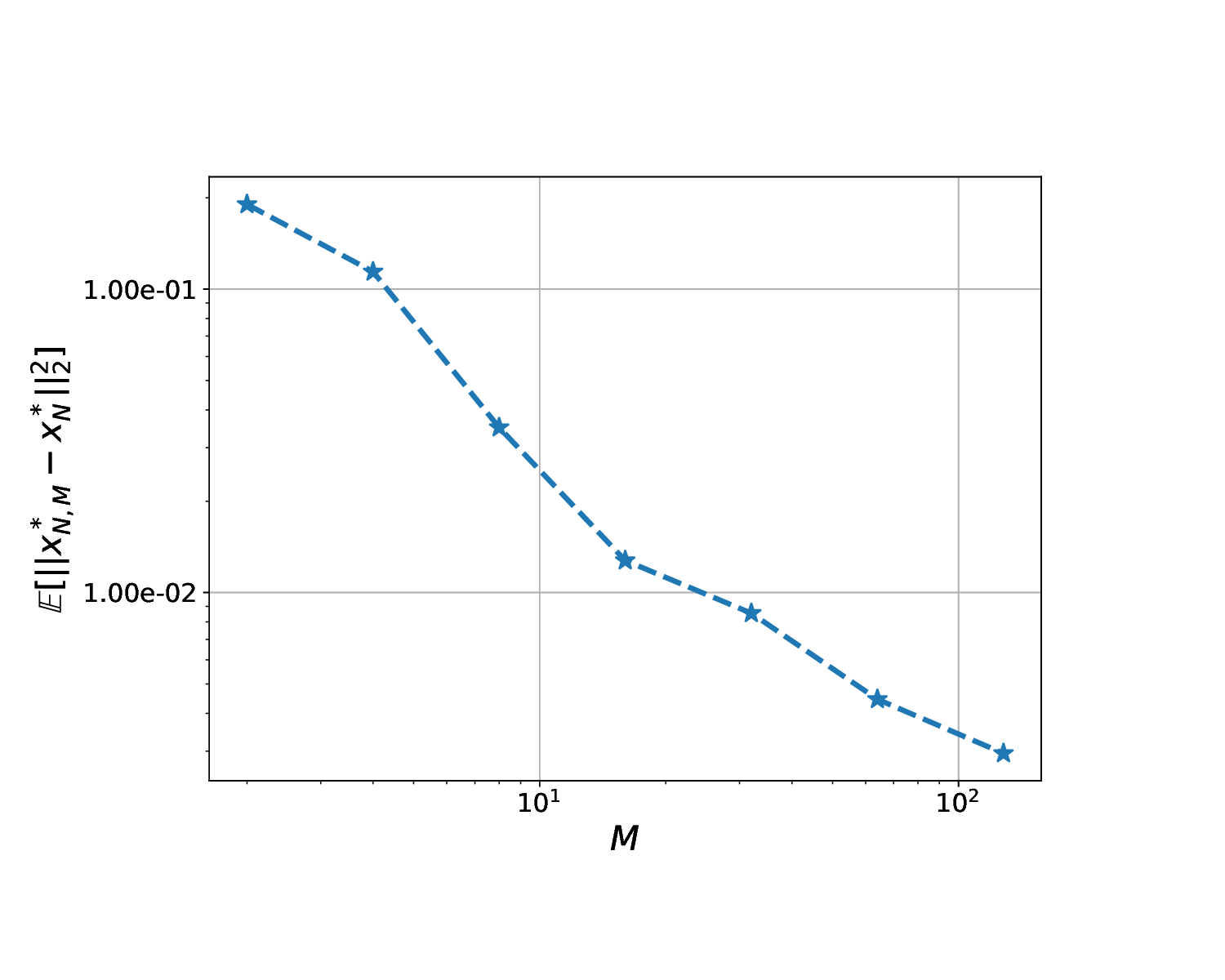}
    \caption{}
    \label{fig:OneEIInner}
    \end{subfigure}
    \caption{Sample average approximation rate of convergence. The Matern kernel is applied with six observations. \subref{fig:OneEIOuter}: convergence with respect to $N$. Rate of regression: -0.92. \subref{fig:OneEIInner}: convergence with respect to $M$ with fixed $N_l = 2^{5}$. Rate of regression: -1.05. 100 realizations are used for both plots.}
\end{figure}

{\subsection{Corollary \ref{cor:value} and Corollary \ref{cor:argmax}\label{sec:NumCoro}}
Figure \ref{fig:MLVal} and \ref{fig:optMLVal} verify the complexity of Corollary \ref{cor:value} and Corollary \ref{cor:argmax} where the rate MSE$^{-1}$ occurs for large cost due to the log penalty. It is possible that the constant of the multilevel value function estimator is smaller than that of the multilevel optimizer estimator, see Figure \ref{fig:1DMSE} and \ref{fig:optMLVal}, but this benefit is not practically significant due to the higher computational cost of the multilevel value function estimator, discussed in Appendix \ref{app:thm1}, and the feature of BO which requires repeatedly solving for new observations.}

\begin{figure}
    \centering
    \begin{subfigure}{0.45\textwidth}
    \includegraphics[width=\linewidth]{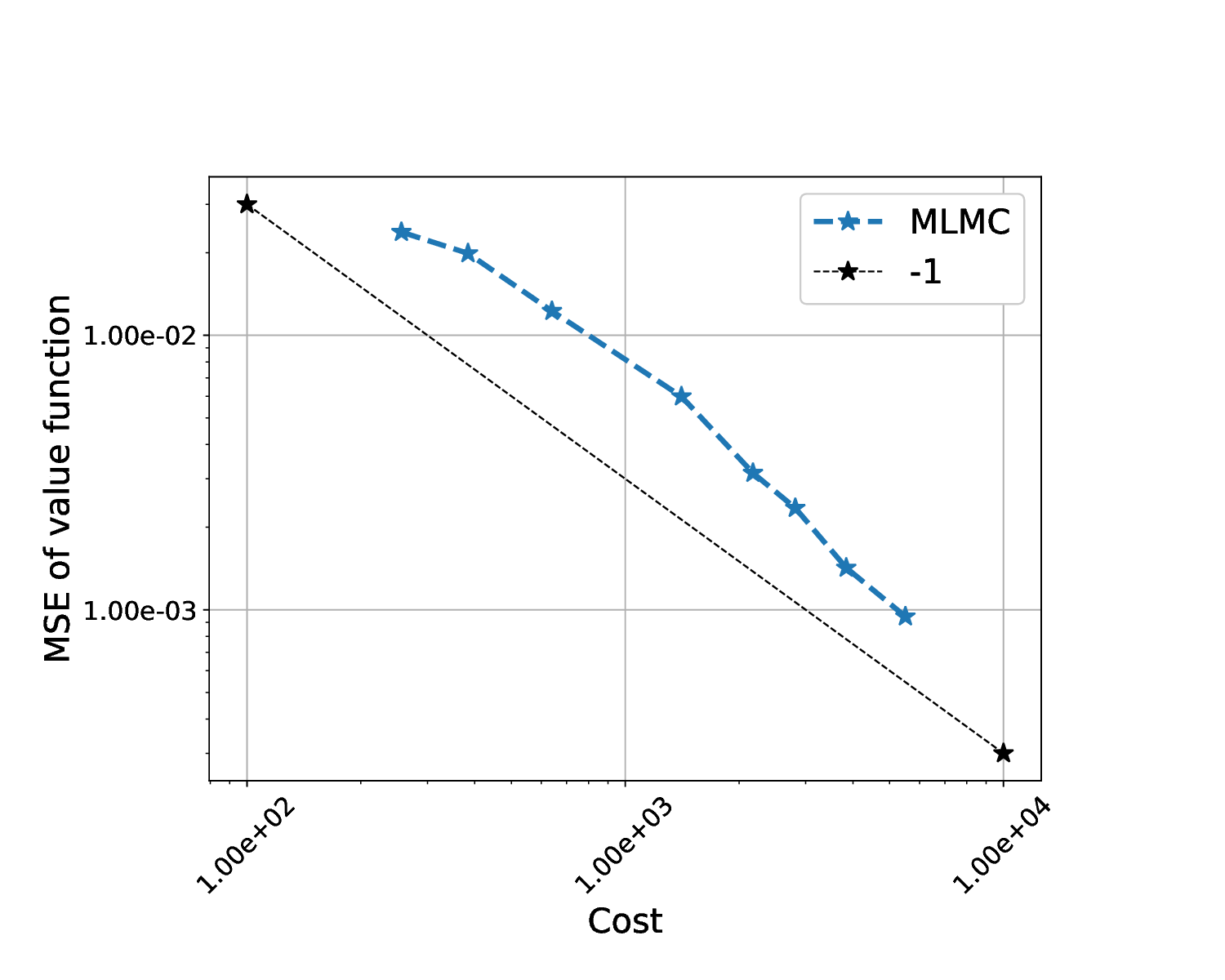}
    \caption{Multilevel value function}
    \label{fig:MLVal}
    \end{subfigure}
    \hfill
    \begin{subfigure}{0.45\textwidth}
    \includegraphics[width=\linewidth]{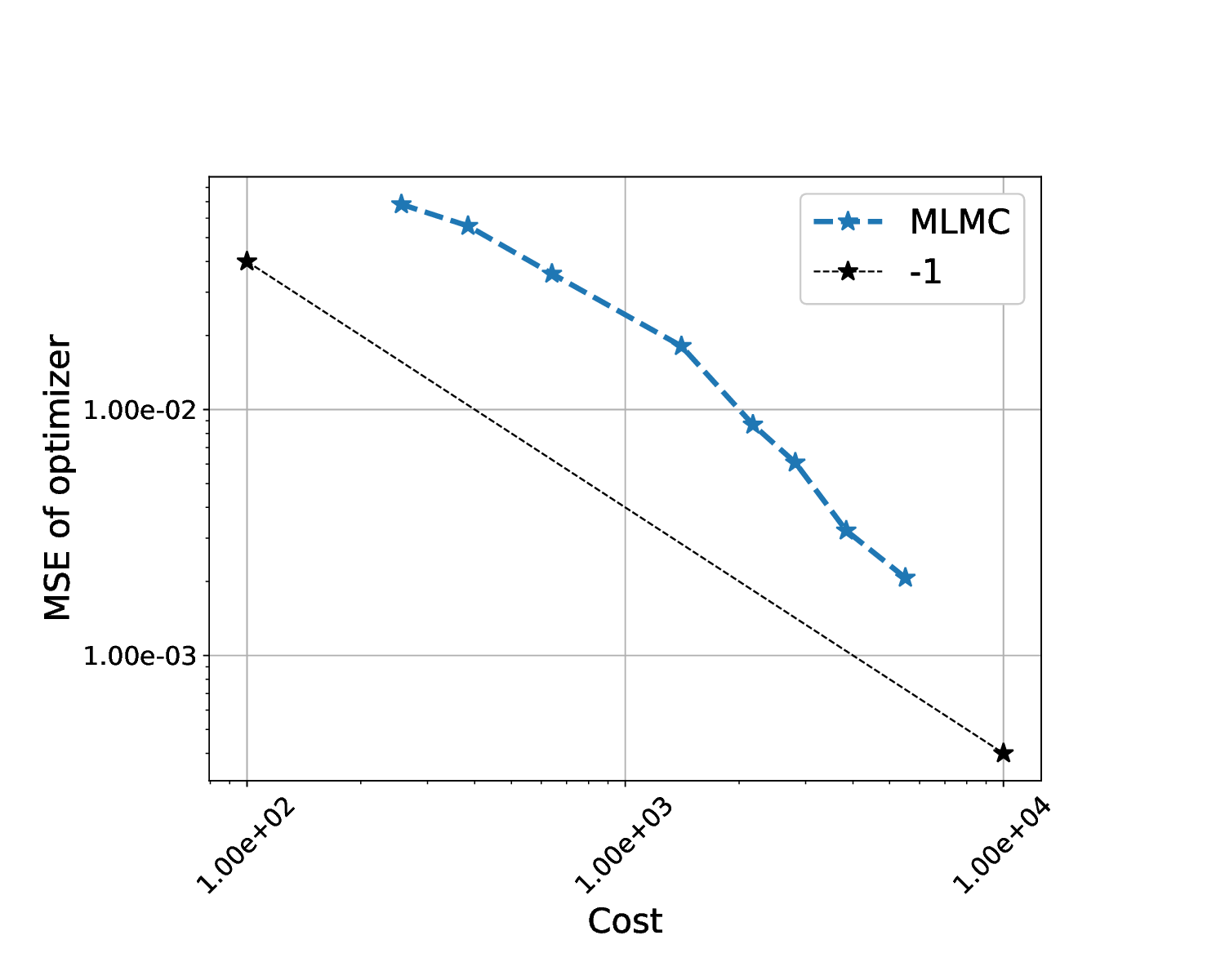}
    \caption{Optimizer of the multilevel value function}
    \label{fig:optMLVal}
    \end{subfigure}
    \caption{Complexity of MLMC value function and the corresponding optimizer. \subref{fig:MLVal}: a fitted slope of -1.08. \subref{fig:optMLVal}: a fitted slope of -1.16. 200 realizations are used for both plots.}
\end{figure}

\subsection{MLMC variance assumptions}\label{sec:NumMLMCRates}
We now numerically verify the multilevel variance assumptions.\\

\noindent
\textbf{Multilevel construction without antithetic approach:}
Figure \ref{fig:MLBeta} shows $\beta = 1$ as we expected.\\

\noindent
\textbf{Multilevel construction with antithetic approach:}
Figure \ref{fig:AntBeta} shows $\beta \approx 1.5$ as we expected.

It is noted that the variance of the increments can be large. 
It is the case where some of the increments can be hundreds of times larger or smaller than others
{\em for small sample sizes}. 
This effect, of course, vanishes asymptotically.
In addition, the constant for the increments is smaller than that of the antithetic increments.

\begin{figure}
    \centering
    \begin{subfigure}{0.45\textwidth}
    \includegraphics[width=\linewidth]{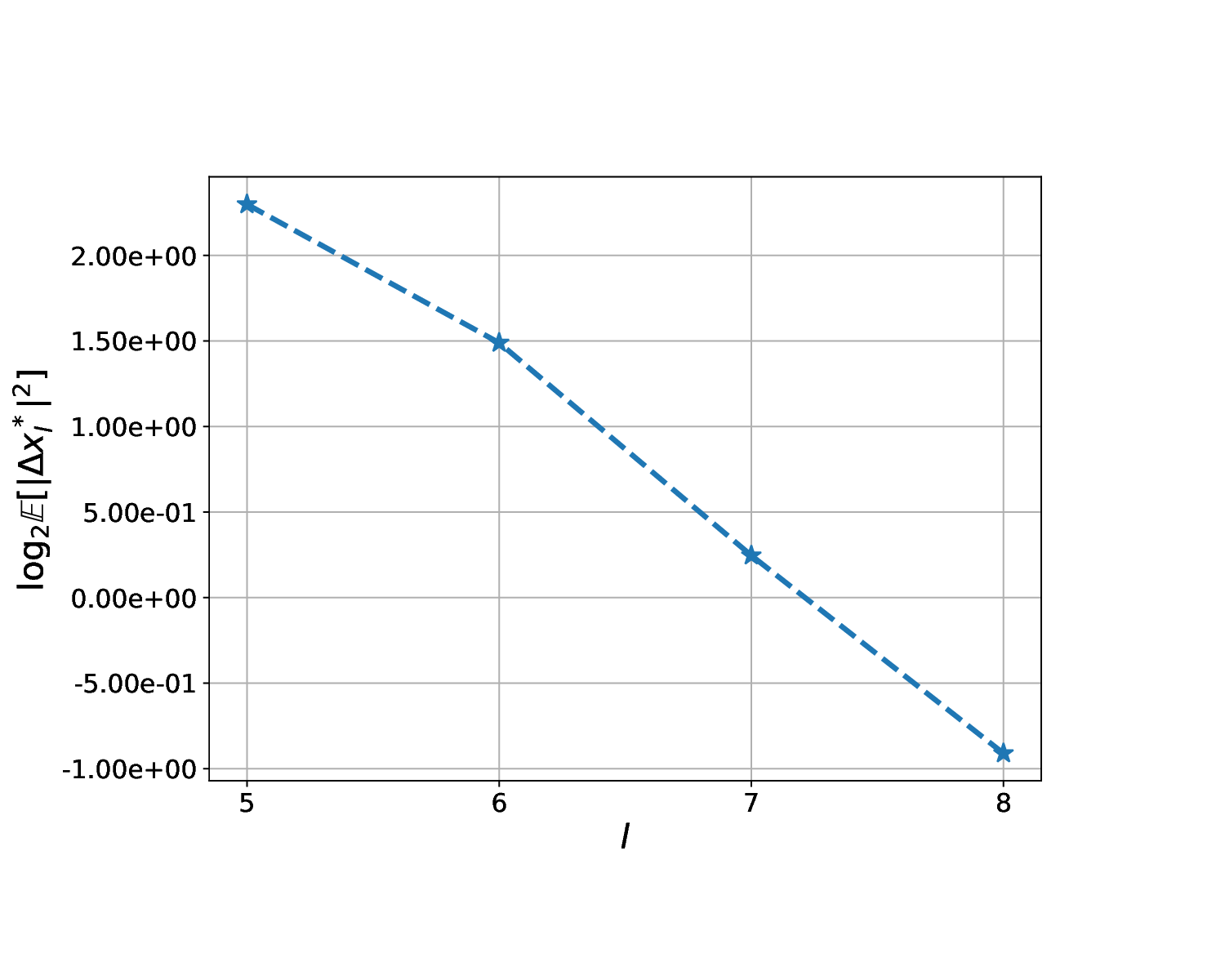}
    \caption{Multilevel increments}
    \label{fig:MLBeta}
    \end{subfigure}
    \hfill
    \begin{subfigure}{0.45\textwidth}
    \includegraphics[width=\linewidth]{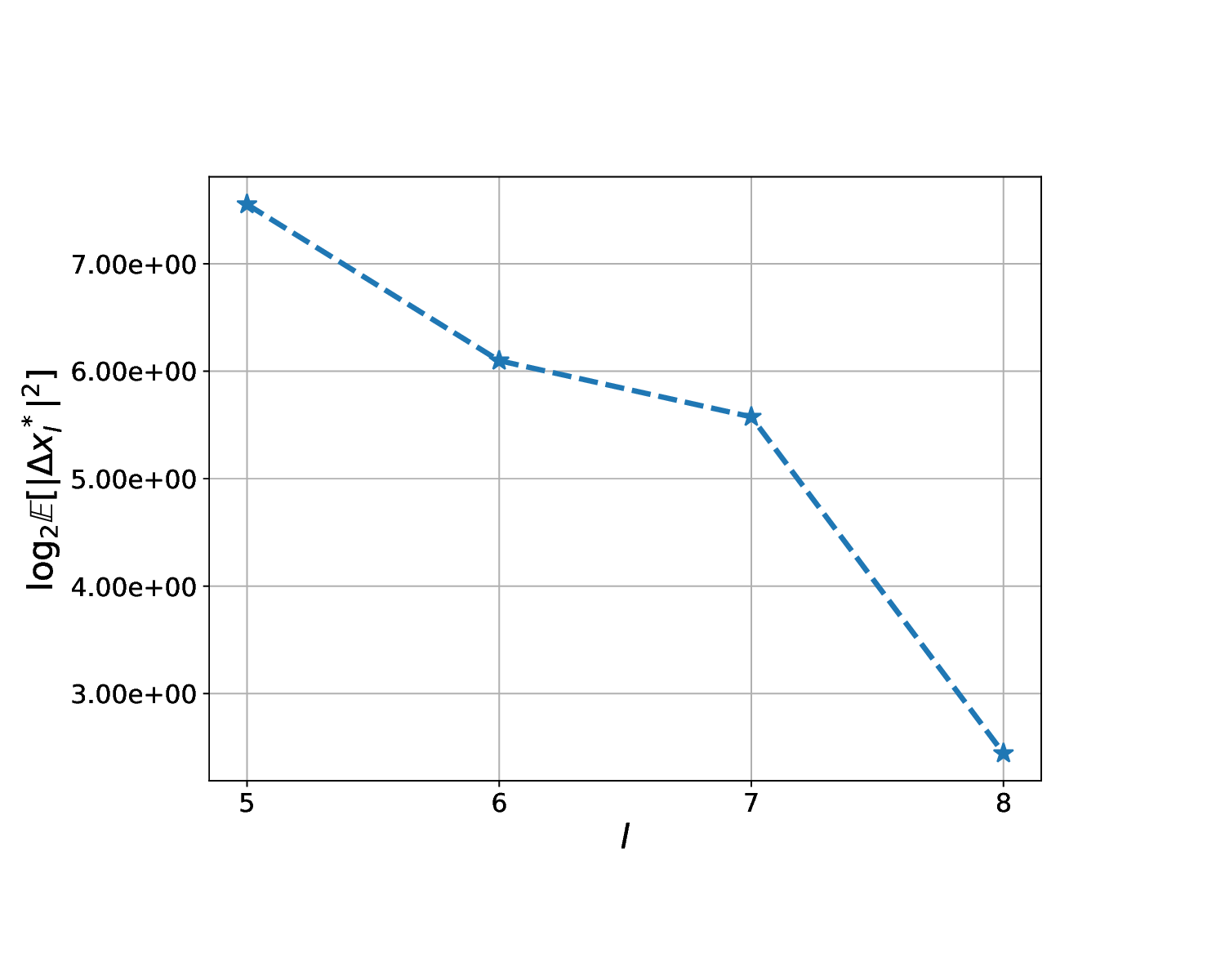}
    \caption{Antithetic multilevel increments}
    \label{fig:AntBeta}
    \end{subfigure}
    \caption{Rate of convergence of increments. \subref{fig:MLBeta}: $\beta$ of the multilevel increments with a fitted slope -1.09. \subref{fig:AntBeta}: $\beta$ of the multilevel antithetic increments with a fitted slope -1.59. A fixed $N_l = 2^{5}$ is applied for all levels, and 50 realisations are used for both plots.}
\end{figure}

\subsection{Results related to Theorem \ref{thm:qFunc}, and Corollaries \ref{cor:value}, \ref{cor:argmax}}
\label{app:thm1}

{Constructing the multilevel value function as Corollary \ref{cor:value} and seeking for optimizer requires solving a higher dimensional joint optimization problem which can be costly than solving multiple low dimensional problems as constructing the multilevel estimator for optimizer directly as Theorem \ref{thm:main} and Remark \ref{rem:Ant}. See Figure \ref{fig:WallTimeVSEps}, which indicates the cost for the multilevel value function is higher than the multilevel optimizer in expectation.
Figure \ref{fig:WallTimeVSObservations} shows computational times of multilevel value function and multilevel optimizer with different numbers of observations. This mimics the whole BO algorithm. A single estimation for the multilevel value function may be more computationally efficient than that for the multilevel optimizer, but in expectation that does not hold. Figure \ref{fig:1DToyqFunc} shows that BO with the multilevel value function estimation even performs worse than BO with MC, but the purpose of introducing the value function estimation is for theoretical analysis and the numerical disadvantage is expected.}

Figure \ref{fig:1DToyqFunc} shows better performance.
(Though use of the antithetic trick may improve performance for value functions.)}

\begin{figure}[htb]
    \centering
        \begin{subfigure}{0.32\textwidth}
    \includegraphics[width=\linewidth]{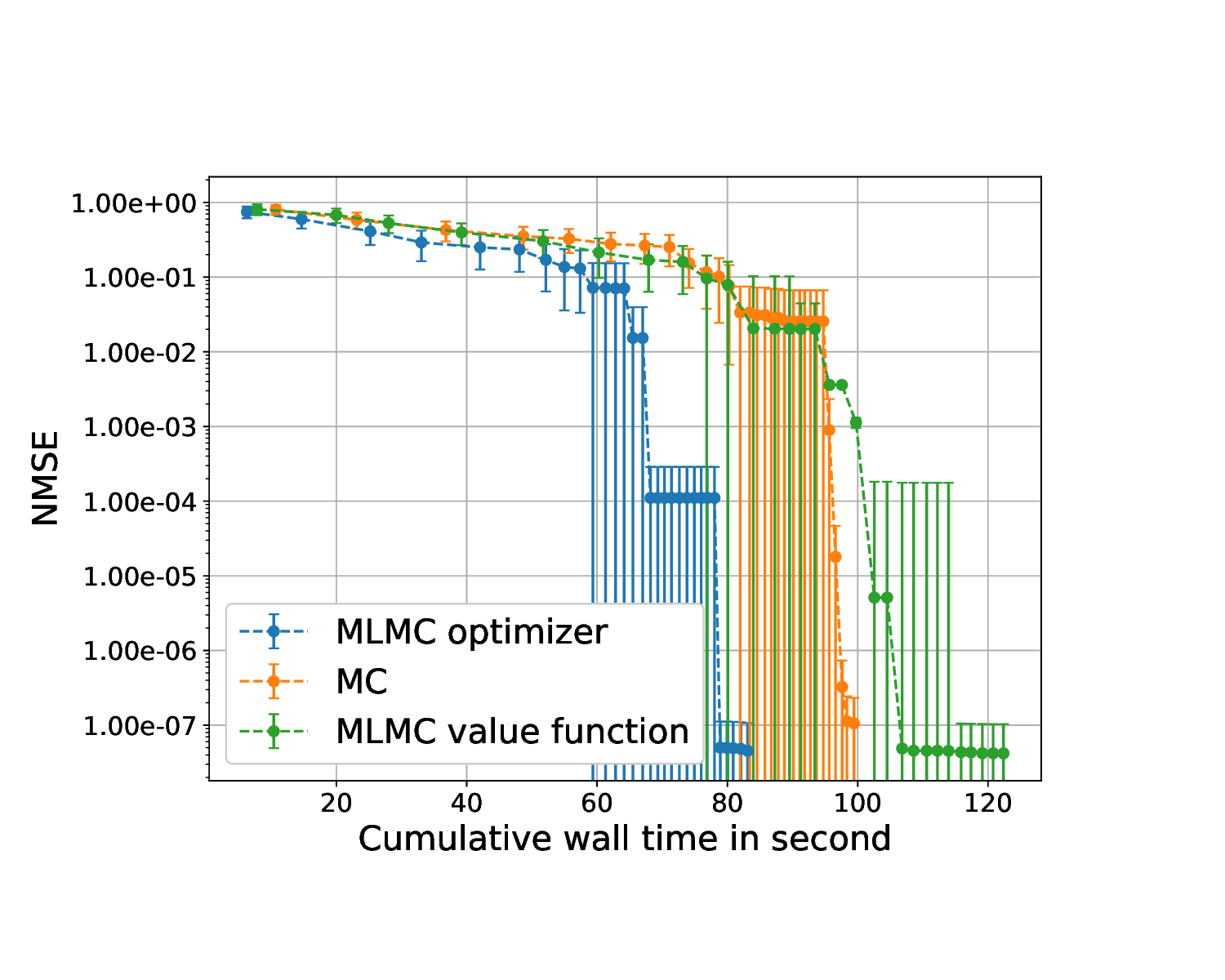}
    \caption{}
    \label{fig:1DToyqFunc}
    \end{subfigure}
    \hfill
    \begin{subfigure}{0.32\textwidth}
    \includegraphics[width=\linewidth]{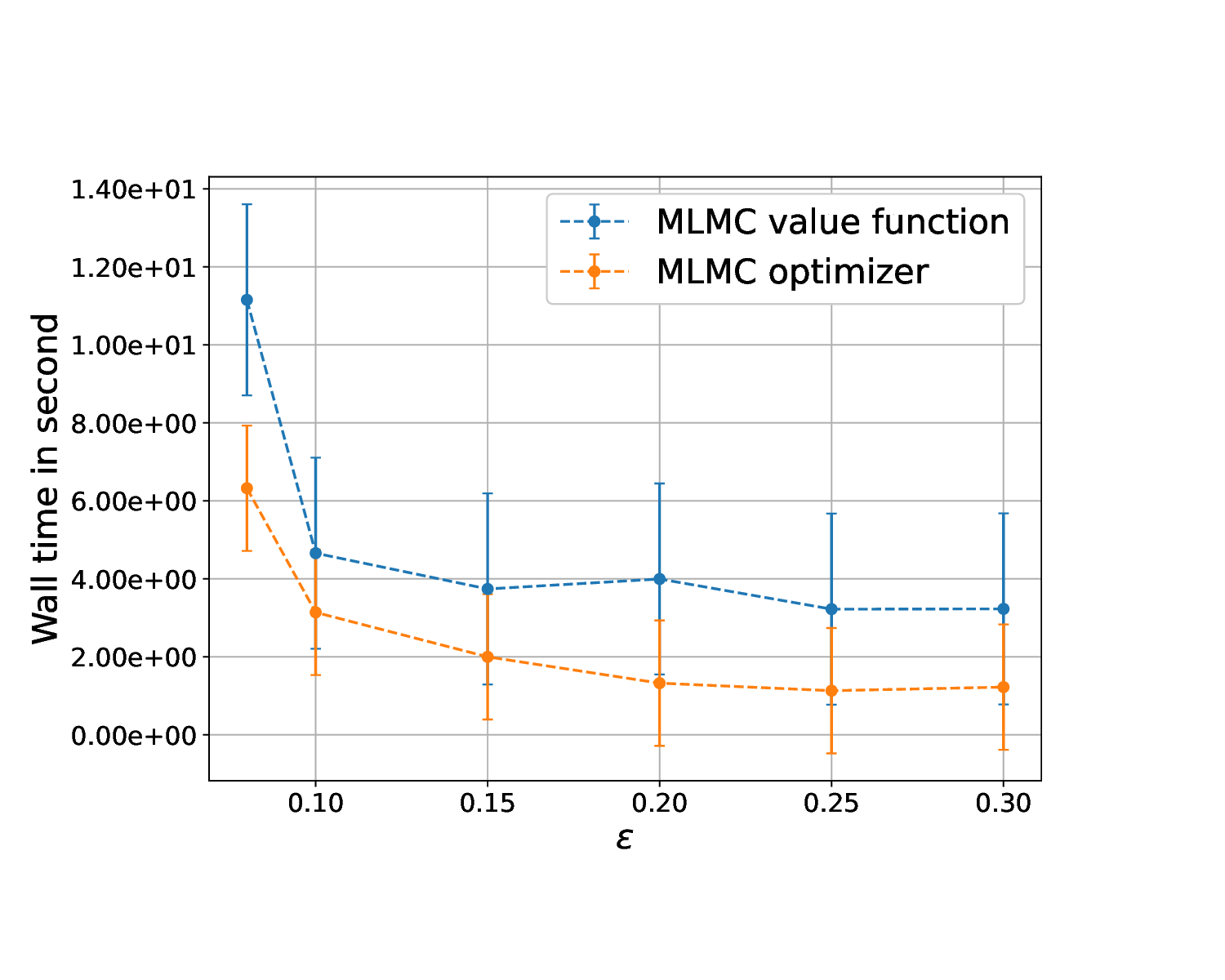}
    \caption{}
    \label{fig:WallTimeVSEps}
    \end{subfigure}
    \hfill
        \begin{subfigure}{0.32\textwidth}
    \includegraphics[width=\linewidth]{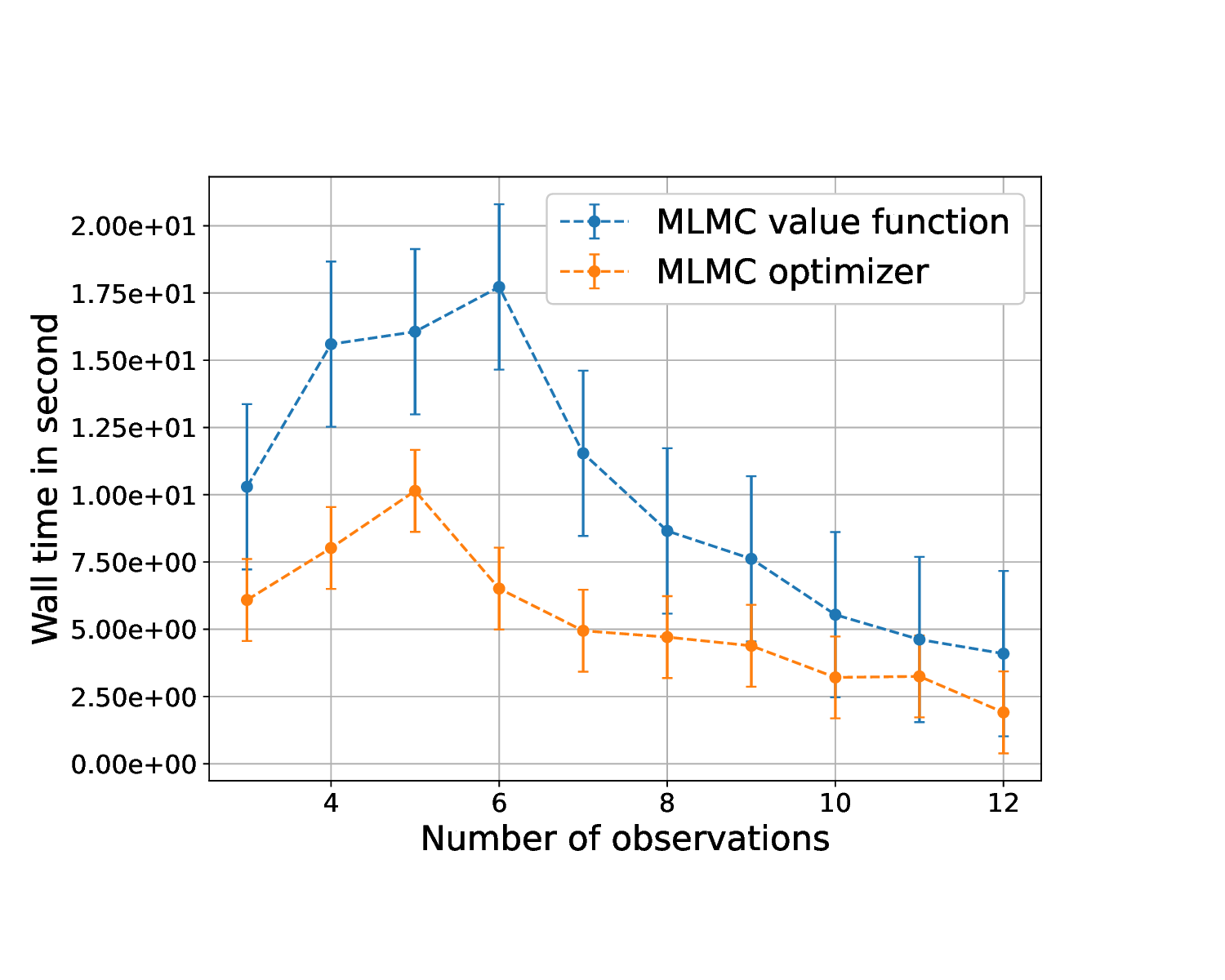}
    \caption{}
    \label{fig:WallTimeVSObservations}
    \end{subfigure}
    \caption{Computational times of multilevel value function and multilevel optimizer. Figure \ref{fig:1DToyqFunc}: convergence of the BO algorithm with respect to the cumulative wall time in seconds, with error bars (computed with 20 realizations). The Matern kernel is applied. The initial BO run starts with 2 observations. Figure \ref{fig:WallTimeVSEps}: time spent by multilevel value function and multilevel optimizer with respect to the required accuracy of the approximation. Figure \ref{fig:WallTimeVSObservations}: time spent by multilevel value function and multilevel optimizer with respect to the different number of observations. The 1D toy example is applied. Each line is computed with 50 realizations}
    \label{fig:TimeCompare}
\end{figure}

\subsection{Further Discussion of Multi-step Look-ahead}\label{app:Compare2OPT}

Figure \ref{fig:Frazier} [reproduced from \cite{wu2019practical}] illustrates the benefit
of the look-ahead acquisition function 2-OPT in comparison to myopic alternatives.
Follow-on work has demonstrated that the value of look-ahead is even
greater for {\em constrained optimization} \cite{lam2017lookahead,frazier2021constraint},
which is a context we intend to explore in the future.
See also \cite{jiang2020efficient} for comparisons between more
multi-step look-ahead candidates and other acquisition functions,
again showing the benefit of being non-myopic. 

We emphasize here that the goal of this paper is to introduce a method that can improve the 
{\em computational efficiency} of multi-step look-ahead acquisition functions, 
either in terms of a wall time for a given accuracy or accuracy for a given wall time.
And that gain is asymptotic, so that {\em the greater the accuracy the greater the gain in efficiency}
(beware the converse that for very low accuracy there may not be any gain).
We leave an exhaustive exploration of this space for the ideal acquisition function to future work.
The recently introduced logEI family are compelling candidates \cite{ament2024unexpected}.

\begin{figure}
    \centering
    \includegraphics[width=.9\linewidth]{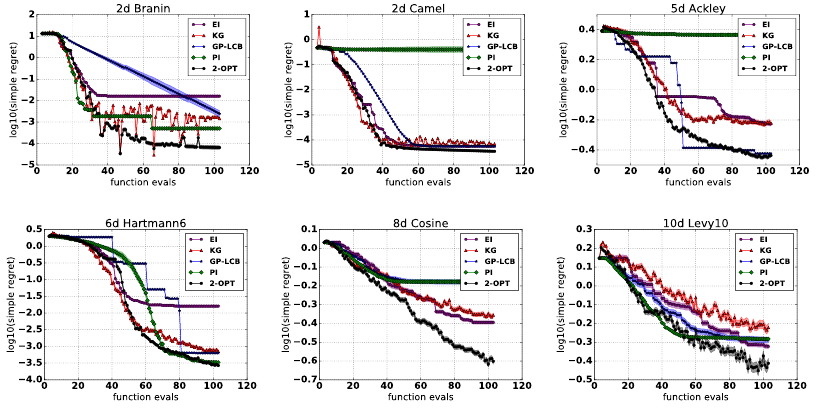}
    \caption{[Reproduced with permission from \cite{wu2019practical}] 
    Synthetic test functions, 90\% quantile of log10 immediate regret compared with common one-step heuristics. 
    2-OPT[, 2-step lookead EI implemented with importance sampling and stochastic gradient descent,]
    provides substantially more robust performance.}
    \label{fig:Frazier}
\end{figure}

\subsection{Results with two-step look-ahead 2-expected improvement \eqref{eq:1qEI}}\label{sec:MLMCReal2EI}
{We now consider the two-step look-ahead 2-EI \eqref{eq:1qEI} where the one- and two-step is a 2-EI. Figure \ref{fig:fullBOMLqEIReal} shows the benefits of using MLMC over MC for 2D Ackley function, 2D DropWave and 2D Shekel function.}

\begin{figure}[htb]
    \centering
    \begin{subfigure}[b]{0.325\textwidth}
    \includegraphics[width=\linewidth]{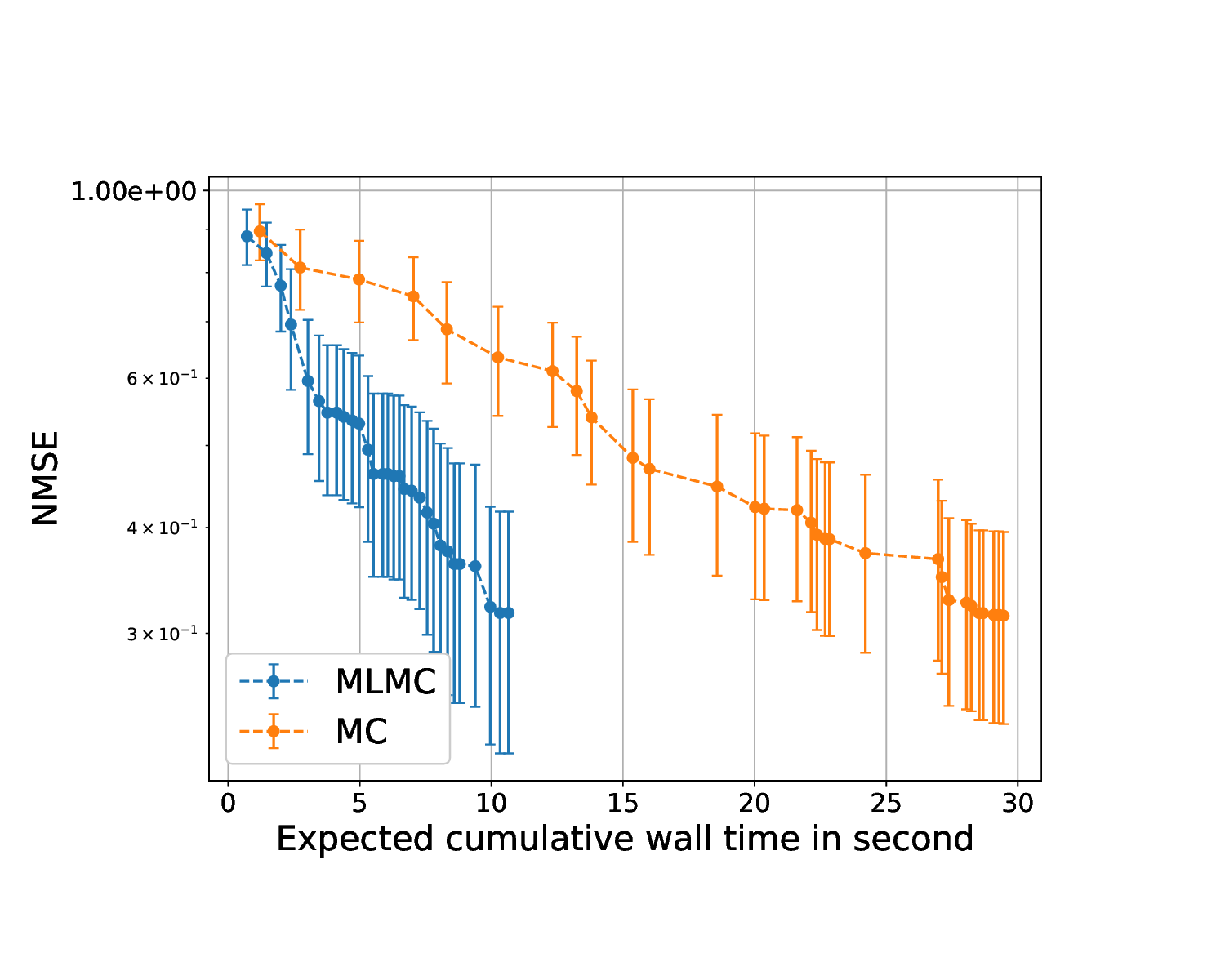}
    \caption{Ackley (d=2)}
    \label{fig:2DAckleyReal}
    \end{subfigure}
     \begin{subfigure}[b]{0.325\textwidth}
    \includegraphics[width=\linewidth]{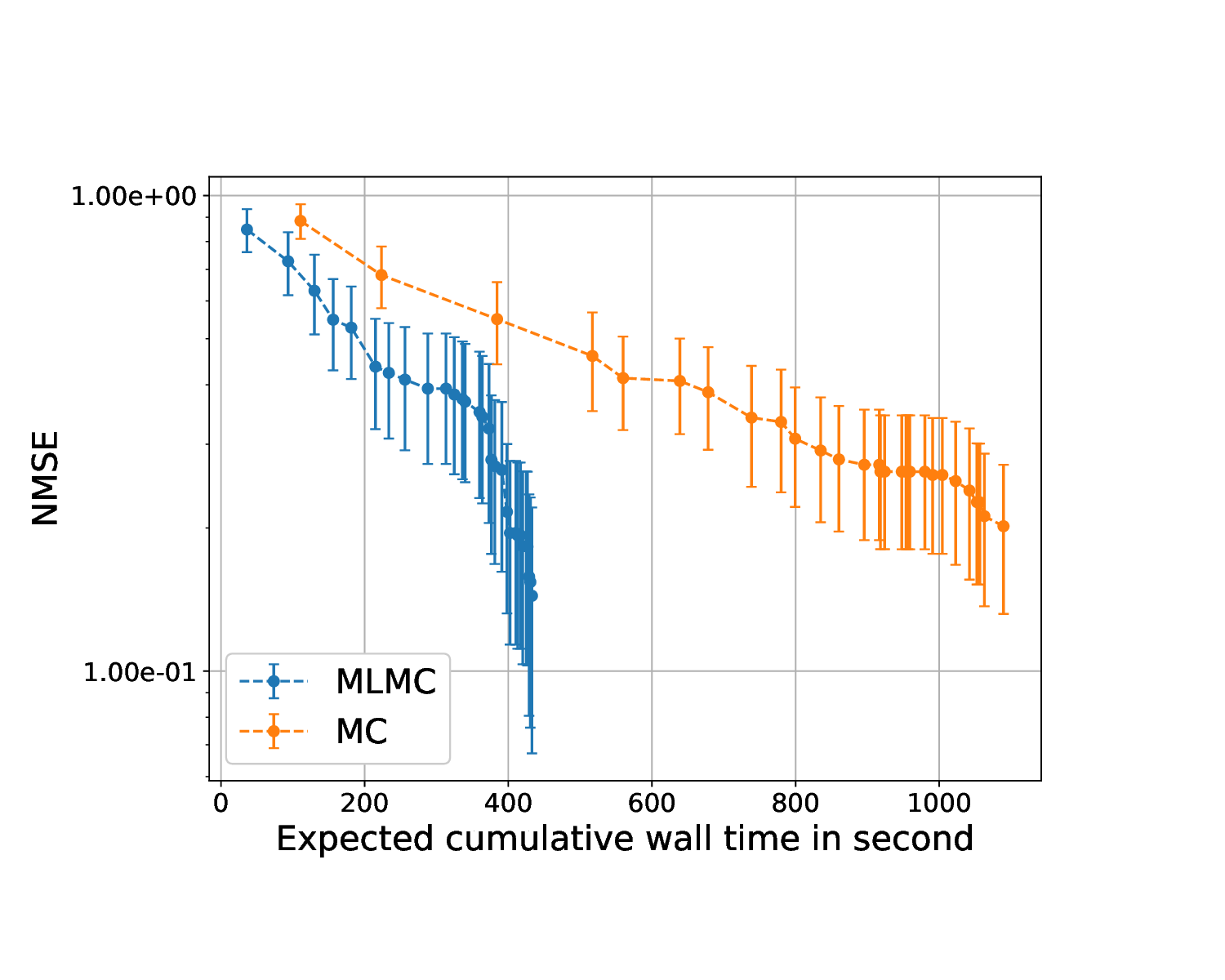}
    \caption{DropWave (d=2)}
    \label{fig:2DDropWave}
    \end{subfigure}
    \begin{subfigure}[b]{0.325\textwidth}
    \includegraphics[width=\linewidth]{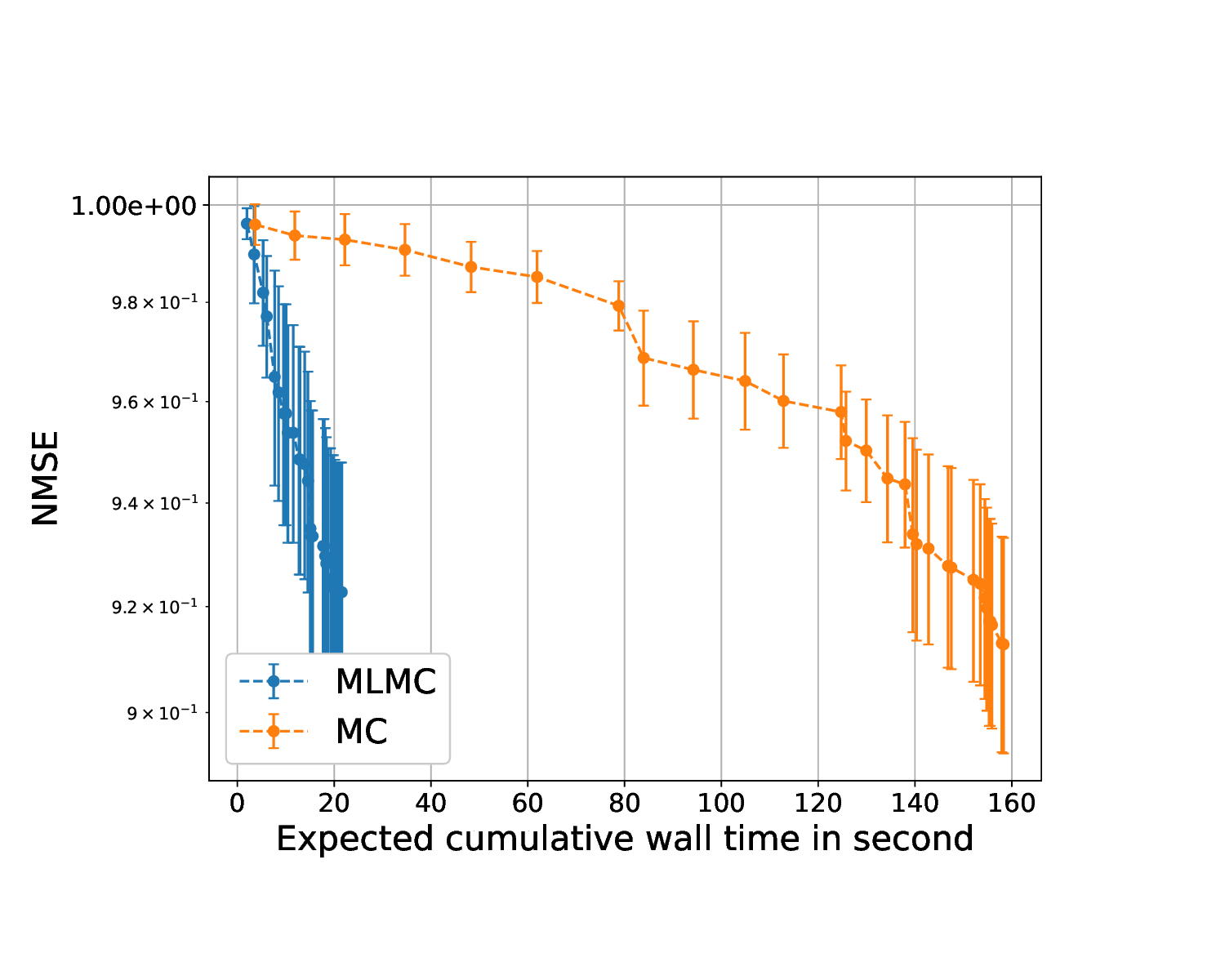}
    \caption{Shekel (d=2)}
    \label{fig:2DShekelReal}
    \end{subfigure}
    \caption{Convergence of the BO algorithm with respect to the cumulative wall time in seconds for 2-step look-ahead 2-EI, with error bars (computed with 20 realizations). The Mat{\'e}rn kernel is applied. The initial BO run starts with $2\times d$ observations.}
    \label{fig:fullBOMLqEIReal}
\end{figure}

\subsection{Matching strategies}\label{sec:Matching}

Two greedy matching strategies are illustrated in Figures \ref{fig:PointMatching} and \ref{fig:BackwardMatching}.

\begin{figure}[htb]
    \centering
    \begin{subfigure}{0.45\textwidth}
    \includegraphics[width=\linewidth]{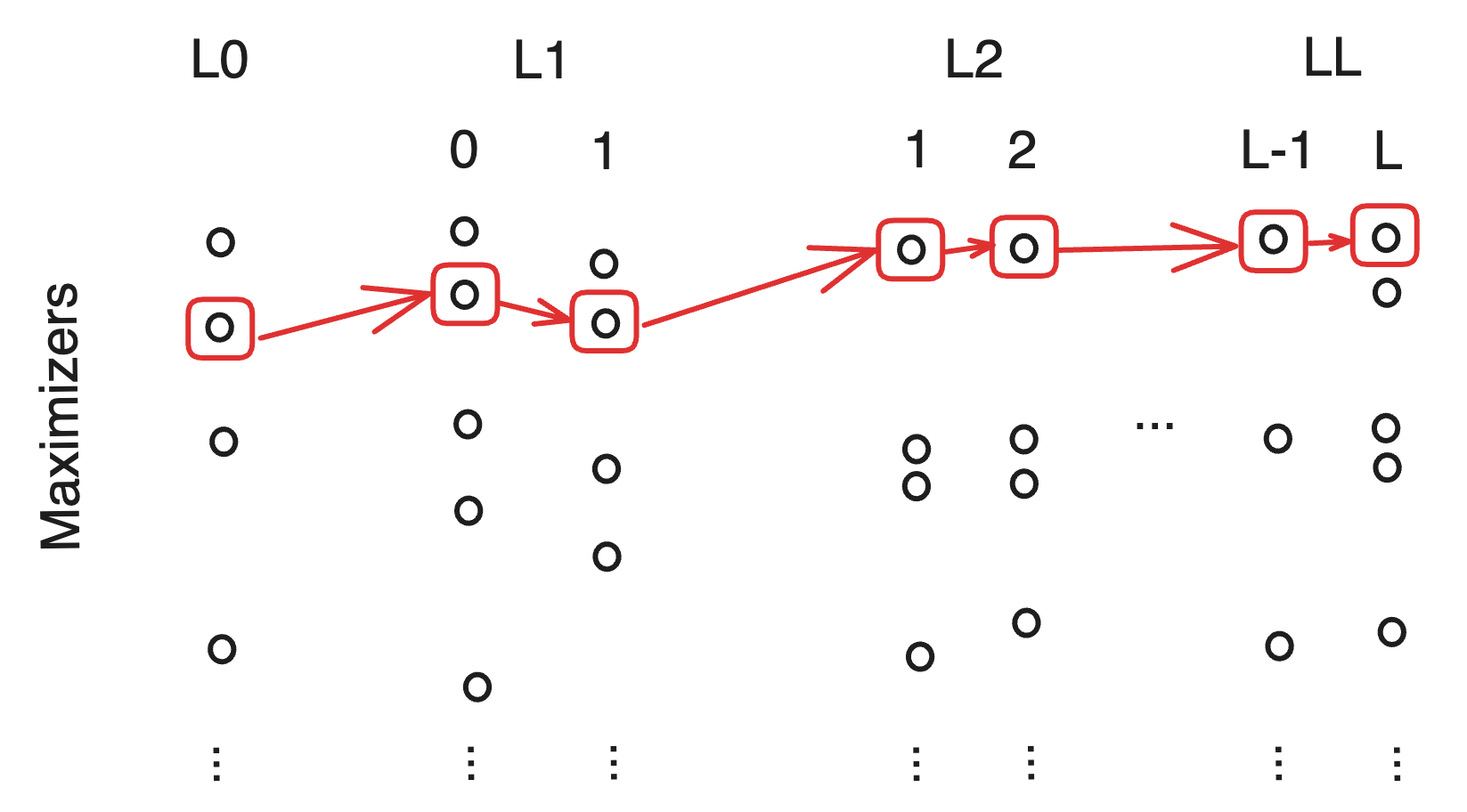}
    \caption{Point matching: all levels align with the best approximation of the global maximizer at level $0$.}
    \label{fig:PointMatching}
    \vspace{.4cm}
    \end{subfigure}
    \hfill
    \begin{subfigure}{0.45\textwidth}
    \includegraphics[width=\linewidth]{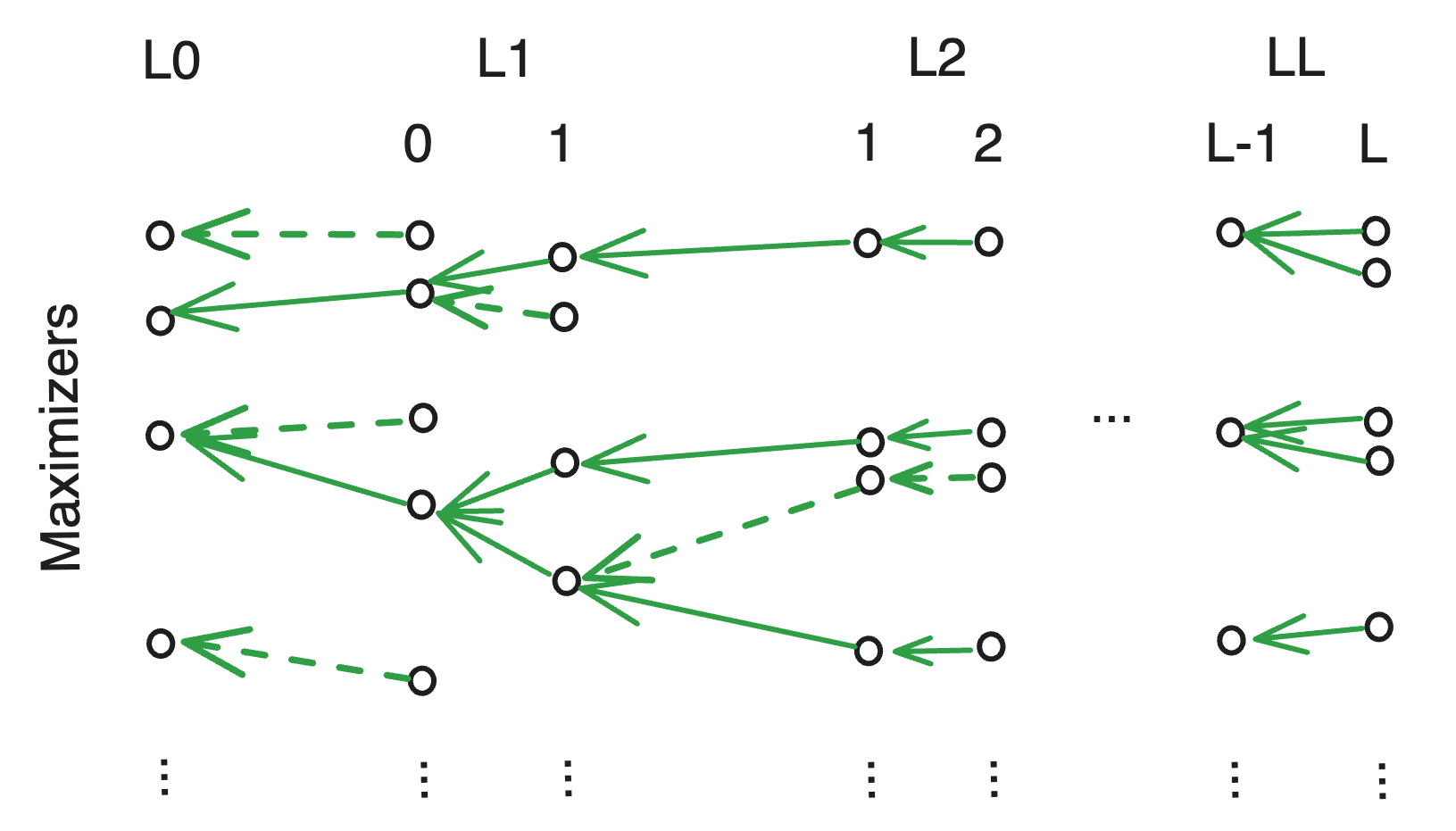}
    \caption{Backward matching: a set of local maximizers are found at each subsequent level and pruned backwards 
    (dashed arrows are pruned branches).}
    \label{fig:BackwardMatching}
    \end{subfigure}
    \caption{Greedy matching strategies. One candidate of the type \eqref{eq:mlmcest} is considered in (a),
    whereas (b)
    results in an estimator of the type \eqref{eq:mlmcest} 
    for each local optimizer found at level $L$ (or $0$, if pruned in the opposite direction).}
\end{figure}

\end{document}